%% file: main.tex
\documentclass{article} 
\usepackage{iclr2027_conference,times}

\input{math_commands.tex}

\usepackage{url}
\usepackage{graphicx}
\usepackage{booktabs}

\title{Optimizing Stochastic Policy Geometry for Few-Step Flow Reinforcement Learning}

\author{%
\makebox[\dimexpr\textwidth-2\tabcolsep\relax][c]{%
\begin{tabular}[t]{@{}c@{}}
Zhuguanyu Wu\textsuperscript{1,2}\thanks{Equal Contribution.}\quad
Congjia Chen\textsuperscript{1}\footnotemark[1]\quad
Ruihao Gong\textsuperscript{1,2}\thanks{Corresponding author.}\quad
Yang Yong\textsuperscript{2}\\[2pt]
Shiqiao Gu\textsuperscript{2}\quad
Yi Xu\textsuperscript{1}\quad
Xianglong Liu\textsuperscript{1}\\[6pt]
\normalfont\textsuperscript{1}Beihang University\quad
\textsuperscript{2}Sensetime Research
\end{tabular}%
}%
}

\newtheorem{proposition}{Proposition}
\newtheorem{theorem}{Theorem}

\usepackage{algorithm}
\usepackage{algpseudocode}
\usepackage{hyperref}
\def\GeoFSRLArxiv{1}

\ifdefined\GeoFSRLArxiv
  \iclrfinalcopy
  \hypersetup{
    hidelinks,
    pdftitle={Optimizing Stochastic Policy Geometry for Few-Step Flow Reinforcement Learning},
    pdfauthor={Zhuguanyu Wu, Congjia Chen, Ruihao Gong, Yang Yong, Shiqiao Gu, Yi Xu, Xianglong Liu}
  }
\fi
\begin{document}

\maketitle
\ifdefined\GeoFSRLArxiv
  \fancyhead{}
  \renewcommand{\headrulewidth}{0pt}
\fi

\begin{abstract}
Few-step flow models generate samples through a handful of large denoising transitions. Applying GRPO to such models requires introducing stochasticity into the denoising transitions, and when a single transition spans a wide range of noise levels, the choice of the resulting stochastic policy becomes a major factor in optimization. Recent constructions constrain the policy to preserve the target flow coefficients, but this constraint admits a family of policies and offers no criterion for choosing among them. We introduce Geo-FSRL (Geometry-Optimized Few-Step Reinforcement Learning), built on a selection principle: at a fixed conditioning state and timestep pair, choose the admissible policy that minimizes the conditional KL divergence induced by any given nonzero velocity-output update. Within a coefficient-preserving family of isotropic Gaussian transitions, this criterion reduces to minimizing a single scalar geometry factor that determines the conditional KL for every velocity-output update. The same factor governs velocity-space Fisher information and conditional transition log-likelihood-ratio variance. We derive the unique closed-form optimizer within this family, shared across all nonzero update magnitudes and directions, with transition coefficients determined by both the source and target noise levels. Because the minimized geometry factor still varies across timesteps, we further adopt gradient-scale matching from prior work and derive closed-form policy-gradient and KL weights that normalize explicit score scales and reference-KL curvature. Geo-FSRL trains on the original few-step schedule and retains the deterministic solver at inference. Our method achieves higher final reward scores than Flow-GRPO and Flow-CPS on both image generation and video generation models under the 4-step setting. Code is available at \url{https://github.com/ModelTC/LightX2V/tree/main/lightx2v_rl}
\end{abstract}

\section{Introduction}

Reinforcement learning (RL) is increasingly used to post-train diffusion and flow models for visual generation, improving visual quality and prompt alignment through reward feedback \citep{black2024training,xue2025dancegrpo,liu2025flow}. A common formulation treats denoising as a sequence of decisions, with intermediate latents as states and rewards assigned to generated samples. For deterministic flow samplers, Flow-GRPO and DanceGRPO introduce stochastic transitions to support exploration and likelihood-ratio optimization \citep{liu2025flow,xue2025dancegrpo}. Meanwhile, distillation compresses generation into a few denoising steps \citep{yin2024one,lightx2v,wu2026sgmd,fan2026phaseddmd}, motivating RL methods that retain this efficient inference schedule. Each transition then spans a wide interval of noise levels, making the finite-step stochastic policy an important design choice.

Large transitions expose differences between stochastic constructions. The standard Euler--Maruyama transition used by Flow-GRPO generally does not satisfy Flow-CPS's coefficient-preserving constraint \citep{liu2025flow,wang2025coefficients}. Flow-CPS enforces this constraint by constructing transitions directly at the target noise level \citep{wang2025coefficients}. We adopt the same constraint, which fixes the sum of the squared noise coefficients while allowing different splits between predicted noise carried over from the current state and fresh Gaussian noise. Flow-CPS specifies this split through a single normalized stochasticity parameter shared across timesteps and tuned empirically. The constraint thus defines an admissible family but leaves a policy-selection question: \emph{which transition should serve as the training policy?}

To answer this question, we introduce Geo-FSRL (Geometry-Optimized Few-Step Reinforcement Learning), motivated by the distinction between the distributions used for policy optimization and the velocity predictions used for generation. GRPO moderates updates through likelihood-ratio clipping and reference-policy regularization, both defined on stochastic transition distributions, whereas deterministic inference depends on velocity predictions. Within the coefficient-preserving isotropic Gaussian family, the same velocity-output update can induce different changes in the transition distribution depending on the noise split. We therefore propose the following selection principle: \emph{at a fixed conditioning state and timestep pair, choose the admissible policy that minimizes the conditional KL divergence induced by any given nonzero velocity-output update}.

To characterize this criterion, we identify a scalar geometry factor that determines the conditional KL for a prescribed velocity-output update. Its square quantifies \emph{policy sensitivity}, measuring how strongly the stochastic transition distribution responds to changes in the velocity prediction. The same factor governs velocity-space Fisher information and conditional log-likelihood-ratio variance under the old policy. This characterization holds for finite updates and yields the same minimizing coefficients for all nonzero update magnitudes and directions, before reward-based optimization determines which updates to take.

The resulting problem has a unique closed-form solution, with coefficients determined by both source and target noise levels. Under rectified-flow parameterization, it recovers the one-step SDE-DPM-Solver++ transition \citep{lu2022dpmsolverplusplus}, providing a geometric optimality characterization of an existing sampling rule. At a fixed target noise level, the minimized geometry factor increases with the transition span, so larger transitions remain more sensitive to the same velocity update even under the optimal split. To account for the remaining variation across timesteps, we follow the gradient-scale matching principle of \citet{meituanlongcatteam2025longcatvideotechnicalreport} to derive policy-gradient and KL weights that normalize explicit score scales and reference-KL curvature.

We evaluate Geo-FSRL on four-step distilled Qwen-Image \citep{wu2025qwen,lightx2v} and Wan2.1 \citep{wan2025wan} models for image and video generation. It attains higher final reward scores than Flow-GRPO and Flow-CPS in the main comparisons while retaining the original deterministic inference schedule.

Our contributions are summarized as follows:
\begin{itemize}
    \item We propose selecting stochastic policies by minimizing conditional KL for a given nonzero velocity-output update, reducing the problem to scalar optimization within a coefficient-preserving isotropic Gaussian family.
    \item We derive the unique closed-form optimizer within this family, shared across all nonzero update magnitudes and directions, and characterize how the minimum geometry factor depends on the transition span at a fixed target noise level.
    \item We develop Geo-FSRL by combining the selected transition with closed-form policy-gradient and KL weights from gradient-scale matching, and demonstrate higher final reward scores than Flow-GRPO and Flow-CPS on four-step image and video models.
\end{itemize}

\section{Preliminaries}

\subsection{Flow Matching and GRPO}

We consider rectified-flow models defined by the linear probability path
\begin{equation}
    \bm{x}_{t}
    =
    (1-t)\bm{x}_{0}
    +
    t\bm{\epsilon},
    \qquad
    \bm{\epsilon}\sim\mathcal{N}(\bm{0},\bm{I}),
    \qquad
    t\in[0,1].
\end{equation}
where $\bm{x}_{0}$ is a clean sample and the model
$\bm{v}_{\theta}(\bm{x}_{t},t)$ is trained to predict the velocity
$\bm{v}=\bm{\epsilon}-\bm{x}_{0}$. Given a prediction, the clean sample
and the noise are estimated as
\begin{equation}
    \hat{\bm{x}}_{0}
    =
    \bm{x}_{t}-t\bm{v}_{\theta}(\bm{x}_{t},t),
    \qquad
    \hat{\bm{\epsilon}}
    =
    \bm{x}_{t}+(1-t)\bm{v}_{\theta}(\bm{x}_{t},t).
    \label{eq:reconstruction_estimates}
\end{equation}
Sampling follows a decreasing timestep schedule
$\mathcal{T}=\{t_{0}=1>t_{1}>\cdots>t_{K}=0\}$ with $K$ denoising
steps; for a transition between timesteps $t>s$ we
write $\Delta_{t,s}=t-s>0$ for its span.

For online RL, denoising is cast as a sequential decision process
\citep{black2024training}: for the transition from $t$ to $s$, the state
is the prompt $\bm{c}$ together with the current latent $\bm{x}_{t}$, and
the next latent $\bm{x}_{s}$ is an action drawn from a stochastic policy
$q_{\theta}(\bm{x}_{s}\mid\bm{x}_{t},\bm{c})$. GRPO \citep{shao2024deepseekmath, liu2025flow} optimizes $\theta$ without a learned value function. For each prompt, a group of $M$ trajectories is
sampled from the current policy $\theta_{\mathrm{old}}$, which is frozen
for one iteration, and the terminal rewards are converted into
group-relative advantages $A$.
The reward-wise normalization and multi-reward combination used in our
experiments are specified in Appendix~\ref{app:experimental_details}.
With the policy ratio
$r_{t,s}(\theta)
=q_{\theta}(\bm{x}_{s}\mid\bm{x}_{t},\bm{c})
/q_{\theta_{\mathrm{old}}}(\bm{x}_{s}\mid\bm{x}_{t},\bm{c})$,
a standard reference-regularized objective is
\begin{equation}
    \begin{aligned}
    \mathcal{L}_{\mathrm{GRPO}}
    ={}&
    -\mathbb{E}\left[
        \min\left(
            r_{t,s}(\theta)A,\,
            \operatorname{clip}\left(
                r_{t,s}(\theta),
                1-\varepsilon_{\mathrm{clip}},
                1+\varepsilon_{\mathrm{clip}}
            \right)A
        \right)
    \right]
    \\
    &+
    \lambda_{\mathrm{KL}}\,
    \mathbb{E}\left[
        D_{\mathrm{KL}}\left(
            q_{\theta}(\cdot\mid\bm{x}_{t},\bm{c})
            \,\Vert\,
            q_{\mathrm{ref}}(\cdot\mid\bm{x}_{t},\bm{c})
        \right)
    \right],
    \end{aligned}
    \label{eq:grpo_objective}
\end{equation}
where expectations run over prompts and group trajectories, and
$q_{\mathrm{ref}}$ is a frozen reference policy, typically the initial
model. The sampling policy $\theta_{\mathrm{old}}$ defines the ratio and
is refreshed every iteration; $q_{\mathrm{ref}}$ anchors training.
Section~\ref{sec:method} analyzes the divergence between
$q_{\theta_{\mathrm{old}}}$ and $q_{\theta}$ induced by one update,
while the reference KL in Eq.~\ref{eq:grpo_objective} regularizes
training.

\subsection{Stochastic Transitions for Few-Step Flow RL}
\label{sec:stochastic_transitions}

Flow matching ordinarily generates samples through a deterministic ODE.
Substituting Eq.~\ref{eq:reconstruction_estimates}, the Euler update
from $t$ to $s$ can be rewritten as
$\bm{x}_{t}-\Delta_{t,s}\bm{v}_{\theta}(\bm{x}_{t},t)
=(1-s)\hat{\bm{x}}_{0}+s\hat{\bm{\epsilon}}$
(Appendix~\ref{sec:app1}). Stochastic transitions used in flow RL inject
Gaussian noise into this update, and those considered in this paper can
all be written in the common form
\begin{equation}
    \bm{x}_{s}
    =
    (1-s)\hat{\bm{x}}_{0}
    +
    \beta_{t,s}\hat{\bm{\epsilon}}
    +
    \alpha_{t,s}\bm{\xi},
    \qquad
    \bm{\xi}\sim\mathcal{N}(\bm{0},\bm{I}),
    \label{eq:common_transition_form}
\end{equation}
with $\bm{\xi}$ sampled independently; $\beta_{t,s}$ scales the predicted
noise and $\alpha_{t,s}$ the newly injected noise.
Appendix~\ref{sec:app1} compares these coefficients in
Table~\ref{tab:transition_coefficients} and gives the derivations and
implementation variants.

\paragraph{Flow-GRPO.}
Flow-GRPO \citep{liu2025flow} constructs a continuous-time SDE whose
marginals coincide with those of the flow ODE and discretizes each
transition with Euler--Maruyama, using
$\sigma_t=a\sqrt{t/(1-t)}$ with a global noise level $a>0$.
For the standard uncapped construction at interior timesteps, the
clean-sample coefficient remains $1-s$, while
$\alpha_{t,s}^2+\beta_{t,s}^2-s^2$ is strictly positive and grows
quadratically with the span at fixed $t$ and $\sigma_t$
(Table~\ref{tab:transition_coefficients}). These are algebraic properties
of the transition coefficients; they do not by themselves determine the
sample marginals under learned predictions or establish a performance
disadvantage. Appendix~\ref{sec:app1} treats the schedule endpoint and
noise cap.

\paragraph{Flow-CPS.}
Flow-CPS \citep{wang2025coefficients} instead constructs the transition
directly at the target noise level, parameterizing the noise split by a
global normalized stochasticity parameter $\phi\in(0,1]$, corresponding
to an angle $\pi\phi/2$:
\begin{equation}
    \alpha_{t,s}=s\sin(\pi\phi/2),
    \qquad \beta_{t,s}=s\cos(\pi\phi/2).
    \label{eq:flow_cps_coefficients}
\end{equation}
These coefficients satisfy
$\alpha_{t,s}^{2}+\beta_{t,s}^{2}=s^{2}$ for every span, with the ODE
Euler update at the deterministic boundary $\phi=0$.
The constraint leaves a one-parameter family of
$(\alpha_{t,s},\beta_{t,s})$: every admissible $\phi$ preserves the
coefficients, and Flow-CPS selects $\phi$ empirically.
Section~\ref{sec:method} optimizes this choice.

\section{Geo-FSRL}
\label{sec:method}

Geo-FSRL selects a stochastic training policy by minimizing the conditional
KL divergence induced by a prescribed nonzero velocity-output update.
We first formulate this criterion, then derive an equivalent scalar
optimization and its closed-form solution. Matched weights subsequently
normalize explicit policy-gradient and reference-regularization scales
without changing the selected policy's likelihood ratios.

\subsection{Admissible Stochastic Policies}
\label{sec:policy_family}

For a transition with $0<s<t\leq1$, we use the common form in
Eq.~\ref{eq:common_transition_form} and impose
\begin{equation}
    \alpha_{t,s}>0,\qquad 0\leq\beta_{t,s}<s,\qquad
    \alpha_{t,s}^{2}+\beta_{t,s}^{2}=s^{2}.
    \label{eq:coefficient_constraint}
\end{equation}
Thus $\beta_{t,s}$ indexes the policies, with
$\alpha_{t,s}=\sqrt{s^{2}-\beta_{t,s}^{2}}$.
Fixing the total noise coefficient prevents trivially reducing
conditional KL by injecting unbounded noise. Under oracle reconstruction,
the predicted and fresh noise are independent standard Gaussians,
reproducing the linear path's noise coefficient at $s$.
Our conditional analysis requires only Gaussian fresh noise $\bm\xi$;
the learned noise prediction need not be standard Gaussian.

Substituting the reconstruction estimates into the transition gives
\begin{equation}
    \begin{aligned}
    q_\beta(\cdot\mid\bm v)
        &=\mathcal N\!\left(\bm\mu_\beta(\bm v),\alpha_{t,s}^{2}\bm I\right),\\
    \bm\mu_\beta(\bm v)
        &=(1-s+\beta_{t,s})\bm x_t+\kappa_{t,s}\bm v,
        \qquad \kappa_{t,s}=-t(1-s)+(1-t)\beta_{t,s}.
    \end{aligned}
    \label{eq:policy_mean_velocity_form}
\end{equation}
Fix a state $(\bm x_t,\bm c)$ and timestep pair $(t,s)$.
Let $\bm v_0$ be the old velocity prediction and
$\Delta\bm v=\bm v_\theta-\bm v_0\in\mathbb R^D$ an output update,
where $D$ is the latent dimension. Each candidate $\beta$ keeps its
coefficients and covariance fixed between old and updated policies;
across candidates, we prescribe the same $\Delta\bm v$ independently
of the optimizer.

\subsection{Policy Selection by Conditional KL}
\label{sec:kl_selection}
\label{sec:policy_sensitivity}

\paragraph{A common measure of model change.}
The velocity output provides a common measure of model change across
stochastic constructions. At the fixed state, the deterministic Euler
update changes by $-(t-s)\Delta\bm v$, independently of the training
noise split. Yet the same $\Delta\bm v$ can change stochastic transition
distributions differently. Coefficient selection therefore determines how
a model-output change is expressed in its training policy.

We seek to limit this distribution change for a prescribed output update.
Within our Gaussian family, smaller conditional KL implies lower
conditional log-ratio variance under each candidate's old policy and,
under the threshold condition in
Section~\ref{sec:geometry_matched_objective}, a lower probability of
leaving the clipping window. Loss reweighting can normalize explicit gradient
and reference-penalty scales but leaves likelihood ratios for a prescribed
output update unchanged, so these differences persist after scale matching.

\paragraph{Policy-selection criterion.}
For each candidate, we measure the distribution change induced by an
output update using the conditional KL divergence between its updated
and old policies:
\begin{equation}
    K_\beta(\Delta\bm v)
    =D_{\mathrm{KL}}\!\left(
        q_\beta(\cdot\mid\bm v_0+\Delta\bm v)
        \,\Vert\,q_\beta(\cdot\mid\bm v_0)\right).
    \label{eq:kl_cost}
\end{equation}
Our selection principle is to \emph{choose the admissible policy that
minimizes the conditional KL divergence induced by the same prescribed
nonzero velocity-output update}. At fixed $(\bm x_t,\bm c,t,s)$ and
$\Delta\bm v\neq\bm0$, this defines the optimization problem
\begin{equation}
    \min_{0\leq\beta_{t,s}<s}\;K_\beta(\Delta\bm v).
    \label{eq:kl_selection_problem}
\end{equation}
The criterion selects a policy for a prescribed update; reward feedback
determines which output updates are useful. We next solve this conditional
KL problem within the admissible family.

\subsection{Geometry Factor and Optimal Transition}
\label{sec:geometry_optimal_policy}
\label{sec:policy_geometry}

\paragraph{A scalar geometry factor $G$.}
To analyze Eq.~\ref{eq:kl_selection_problem}, note that an output update
shifts the policy mean by $\kappa_{t,s}\Delta\bm v$ while leaving its
covariance fixed. We define
\begin{equation}
    G_{t,s}(\beta_{t,s})
    =\frac{|\kappa_{t,s}|}{\alpha_{t,s}}
    =\frac{t(1-s)-(1-t)\beta_{t,s}}
          {\sqrt{s^{2}-\beta_{t,s}^{2}}}.
    \label{eq:geometry_factor}
\end{equation}
This \emph{geometry factor} is the standardized policy-mean displacement
per unit velocity change:
$\|\Delta\bm\mu\|_2/\alpha_{t,s}=G_{t,s}\|\Delta\bm v\|_2$.
It depends only on the timestep pair and coefficient choice, and is
strictly positive on the admissible family.

\begin{proposition}[Equivalence to the conditional KL criterion]
\label{prop:policy_sensitivity}
At a fixed state and timestep pair, every admissible coefficient choice
and every finite output update satisfy
\begin{equation}
    K_\beta(\Delta\bm v)
        =\frac12G_{t,s}^{2}\|\Delta\bm v\|_2^2.
    \label{eq:policy_kl_geometry}
\end{equation}
For any prescribed $\Delta\bm v\neq\bm0$, the selection problem in
Eq.~\ref{eq:kl_selection_problem} is therefore equivalent to minimizing
$G_{t,s}$:
\begin{equation}
    \operatorname*{arg\,min}_{0\leq\beta_{t,s}<s} K_\beta(\Delta\bm v)
    =\operatorname*{arg\,min}_{0\leq\beta_{t,s}<s} G_{t,s}^2(\beta_{t,s})
    =\operatorname*{arg\,min}_{0\leq\beta_{t,s}<s} G_{t,s}(\beta_{t,s}).
    \label{eq:geometry_optimization_problem}
\end{equation}
In particular, its minimizing coefficients are independent of the
prescribed update's nonzero magnitude and direction.
\end{proposition}

\noindent\emph{Proof.}
Within each candidate, the old and updated policies have equal covariance
$\alpha_{t,s}^2\bm I$ and mean difference $\kappa_{t,s}\Delta\bm v$.
The Gaussian KL formula gives
$K_\beta=\|\kappa_{t,s}\Delta\bm v\|_2^2/(2\alpha_{t,s}^2)$,
which is Eq.~\ref{eq:policy_kl_geometry}. Since
$\|\Delta\bm v\|_2^2/2$ is a positive constant across candidates and
$G_{t,s}>0$, the three sets of minimizers coincide. \hfill$\square$

The equivalence is exact for finite updates; its direction independence
follows from isotropic covariance and scalar velocity coupling
(Appendix~\ref{app:policy_geometry}).

\paragraph{$G$ can be minimized in closed form.}
The equivalence above reduces policy selection to minimizing $G_{t,s}$
over the admissible noise split. The following theorem gives the optimal
coefficients and the corresponding minimum geometry factor.
\begin{theorem}[Geometry-optimal transition]
\label{thm:optimal_transition}
For every $0<s<t\leq1$ and prescribed $\Delta\bm v\neq\bm0$,
the selection problem in Eq.~\ref{eq:kl_selection_problem} has the
unique solution
\begin{equation}
    \beta_{t,s}^{\star}=\frac{s^2(1-t)}{t(1-s)},
    \qquad \alpha_{t,s}^{\star}=\sqrt{s^2-(\beta_{t,s}^{\star})^2},
    \qquad G_{t,s}^{\star}=\frac{\sqrt{(t-s)(t+s-2ts)}}{s}.
    \label{eq:optimal_transition_coefficients}
\end{equation}
\end{theorem}

See Appendix~\ref{app:optimal_transition} for the proof.
In Flow-CPS's parameterization,
$\cos(\pi\phi_{t,s}^{\star}/2)=s(1-t)/(t(1-s))$,
so each timestep pair determines its own noise split. Within the
admissible family, this rule replaces empirical selection of a global
$\phi$ with transition-wise coefficients computed directly from $(t,s)$.
At $t=1$, $\beta_{1,s}^{\star}=0$ and $\alpha_{1,s}^{\star}=s$,
allocating all target noise to fresh noise.
As $t\downarrow s$ at fixed $s\in(0,1)$,
$\beta_{t,s}^{\star}\to s$ and $\alpha_{t,s}^{\star}\to0$.
For $t<1$, the optimizer balances mean sensitivity and injected noise:
$\alpha^{\star}_{t,s}<s$, whereas maximum noise occurs at $\beta=0$.
Since $G_{t,s}(0)^2-(G^{\star}_{t,s})^2=(1-t)^2$,
these choices coincide only at $t=1$.
The coefficients recover the one-step SDE-DPM-Solver++ transition
\citep{lu2022dpmsolverplusplus}, equivalently the posterior-variance
ancestral transition under the matching marginal parameterization
(Appendix~\ref{app:dpm_correspondence}), now characterized as optimal
for the stated criterion.

\paragraph{$G$ governs policy sensitivity to velocity updates.}
For any nonzero $\Delta\bm v$,
\begin{equation}
    \frac{2K_\beta(\Delta\bm v)}{\|\Delta\bm v\|_2^2}
    =G_{t,s}^2,
    \qquad \mathcal I_{\bm v}=G_{t,s}^2\bm I,
    \label{eq:normalized_policy_sensitivity}
\end{equation}
where $\mathcal I_{\bm v}$ is the conditional Fisher information with
respect to the velocity output (Appendix~\ref{app:policy_geometry}).
Here $G_{t,s}^2$ quantifies \emph{policy sensitivity}: twice the
conditional KL per squared velocity-update norm.

GRPO evaluates output changes through transition likelihood ratios.
For a prescribed $\Delta\bm v$, consider
$r=q_\beta(\bm x_s\mid\bm v_0+\Delta\bm v)/
q_\beta(\bm x_s\mid\bm v_0)$ with
$\bm x_s\sim q_\beta(\cdot\mid\bm v_0)$. Its conditional log-ratio
variance satisfies
\begin{equation}
    \operatorname{Var}_{q_\beta(\cdot\mid\bm v_0)}(\log r)
    =G_{t,s}^{2}\|\Delta\bm v\|_2^{2}
    =2K_\beta(\Delta\bm v).
    \label{eq:conditional_log_ratio_variance}
\end{equation}
Thus the same coefficient choice minimizes both the conditional KL and
the log-ratio variance for the same output update, with actions sampled
from each candidate's old policy. Appendix~\ref{app:ratio_moments} derives
the full log-ratio distribution and ratio moments.
Section~\ref{sec:geometry_matched_objective} connects this log-ratio
characterization to the dimension-averaged ratio and clipping window used in training.

\paragraph{$G^{\star}$ increases with transition span at fixed $s$.}
Even after optimization, $G_{t,s}^{\star}$ increases strictly with $t$ at
fixed $s\in(0,1)$, from order $\sqrt{t-s}$ for small spans to
$G_{1,s}^{\star}=(1-s)/s$ at the pure-noise endpoint
(Appendix~\ref{app:optimal_transition}). Larger spans therefore increase
the minimum conditional KL and log-ratio variance induced by the same
velocity update under this fixed-target comparison.

\subsection{Geometry-Matched GRPO Training}
\label{sec:geometry_matched_objective}

\paragraph{Transition-wise weights.}
The geometry factor also enters the explicit score scale and
reference-KL curvature. We normalize these using the gradient-scale
matching principle of
\citet{meituanlongcatteam2025longcatvideotechnicalreport}.
At a fixed action, the velocity-output score is
\begin{equation}
    \nabla_{\bm v}\log q_\theta
    =\frac{\kappa_{t,s}^{\star}}{\alpha_{t,s}^{\star}}\bm\xi_\theta,
    \qquad
    \bm\xi_\theta=\frac{\bm x_s-\bm\mu_{\theta,t,s}}
                        {\alpha_{t,s}^{\star}}.
    \label{eq:policy_gradient_velocity}
\end{equation}
The explicit coefficient has magnitude $G_{t,s}^{\star}$. The current
and reference policies use the same coefficients and covariance, so their
conditional KL has velocity-space curvature $(G_{t,s}^{\star})^2\bm I$.
These factors vary across timestep pairs even after coefficient selection.
We therefore use
\begin{equation}
    w_{t,s}^{\mathrm{policy}}=\frac{c_{\mathrm{geo}}}{G_{t,s}^{\star}},
    \qquad w_{t,s}^{\mathrm{KL}}
        =\left(\frac{c_{\mathrm{geo}}}{G_{t,s}^{\star}}\right)^2,
    \label{eq:geometry_matched_weights}
\end{equation}
where $c_{\mathrm{geo}}>0$ is shared across transitions and the weights
are constant with respect to model parameters. The derivation is given in
Appendix~\ref{app:geometry_matched_scaling}.

\paragraph{Training objective.}
Following Flow-GRPO \citep{liu2025flow}, we average log-densities over the
$D$ scalar elements of the transition latent, i.e., all non-batch
dimensions. Thus the clipped term uses $\bar r_{t,s}=r_{t,s}^{1/D}$ and
the reference regularizer uses $\overline D_{\mathrm{KL}}=D_{\mathrm{KL}}/D$:
\begin{equation}
    \mathcal L_{\mathrm{geo}}
    =\mathbb E\!\left[
        w_{t,s}^{\mathrm{policy}}\bar\ell_{t,s}^{\mathrm{clip}}(\theta)
        +\lambda_{\mathrm{KL}}w_{t,s}^{\mathrm{KL}}
         \overline D_{\mathrm{KL}}(q_\theta\,\Vert\,q_{\mathrm{ref}})
        \right].
    \label{eq:geometry_matched_grpo_objective}
\end{equation}
Here $\bar\ell^{\mathrm{clip}}$ is the negative clipped surrogate in
Eq.~\ref{eq:grpo_objective}, with $r$ replaced by $\bar r$, and both
policies are conditioned on the recorded state. At fixed $D$, dimension
averaging changes the KL scale but leaves the minimizing coefficients
and the relative $1/G^{\star}$ and $1/(G^{\star})^2$ weights unchanged.
Advantage processing and hyperparameters are specified in
Appendix~\ref{app:experimental_details}.

\paragraph{Policy sensitivity after scale matching.}
For any candidate policy, applying its matched weight to the
dimension-averaged velocity score at $\theta=\theta_{\mathrm{old}}$
gives $-c_{\mathrm{geo}}\bm\xi_{\mathrm{old}}/D$.
The weighted reference KL is
$c_{\mathrm{geo}}^2\|\bm v_\theta-\bm v_{\mathrm{ref}}\|_2^2/(2D)$,
so its penalty for a fixed velocity deviation is also identical across
matched candidates. These cancellations normalize explicit loss scales.
They do not change each candidate's sampling distribution, likelihood
ratio, or clipping thresholds.

For a prescribed output update at a fixed state, timestep pair, and
dimension, let $d=G_{t,s}\|\Delta\bm v\|_2$. With an action drawn from
the candidate's old policy, the dimension-averaged log-ratio satisfies
\begin{equation}
    \log\bar r\sim\mathcal N\!\left(-\frac{d^2}{2D},\frac{d^2}{D^2}\right).
    \label{eq:reweight_ratio_sensitivity}
\end{equation}
Thus minimizing $G_{t,s}$ still minimizes the conditional KL and
log-ratio variance induced by the same output update after matching.
It also minimizes the probability that $\bar r$ lies outside
$[1-\varepsilon_{\mathrm{clip}},1+\varepsilon_{\mathrm{clip}}]$
whenever
$-\log(1-\varepsilon_{\mathrm{clip}})\leq
3\log(1+\varepsilon_{\mathrm{clip}})$,
as holds for our $\varepsilon_{\mathrm{clip}}=10^{-3}$
(Proposition~\ref{prop:clipping_window}, Appendix~\ref{app:clipping_window}).
This gives a direct connection between coefficient selection and the
clipping window used in training.

This probability compares prescribed updates under the respective
old-policy distributions. Actual updates also depend on sampled
advantages, residuals, model Jacobians, and the optimizer, while active
clipping additionally depends on the advantage sign. Scale matching
therefore does not guarantee identical realized updates or a lower
observed clipping frequency. For a recorded action reused after an
update, $\bm\xi_\theta=\bm\xi_{\mathrm{old}}
-(\kappa_{t,s}/\alpha_{t,s})\Delta\bm v$; the resulting dependence of
later gradients on the coefficients is detailed in
Appendix~\ref{app:geometry_matched_scaling}.

\paragraph{Effect on the sampling distribution.}
At the same state and velocity prediction, changing the coefficients
generally changes both the conditional mean and fresh-noise variance:
$\bm\mu_{\beta_2}-\bm\mu_{\beta_1}
=(\beta_2-\beta_1)\hat{\bm\epsilon}$.
A global learning-rate or clipping adjustment leaves this sampling
distribution unchanged. Moreover, geometry-factor ratios generally vary
across transitions. Coefficient selection therefore changes the policy
beyond a uniform rescaling of its optimization steps.

The sampling and training procedure is given in
Algorithm~\ref{alg:geometry_grpo}
(Appendix~\ref{app:training_algorithm}).

\newlength{\savedfloatsep}
\newlength{\savedtextfloatsep}
\newlength{\savedintextsep}
\setlength{\savedfloatsep}{\floatsep}
\setlength{\savedtextfloatsep}{\textfloatsep}
\setlength{\savedintextsep}{\intextsep}
\setlength{\floatsep}{6pt plus 1pt minus 1pt}
\setlength{\textfloatsep}{10pt plus 1pt minus 2pt}
\setlength{\intextsep}{8pt plus 1pt minus 1pt}
\input{sec/4_exp}
\input{sec/5_related}
\input{sec/6_conclusion}

\bibliography{iclr2027_conference}
\bibliographystyle{iclr2027_conference}

\appendix
\raggedbottom
\clearpage
\setlength{\floatsep}{\savedfloatsep}
\setlength{\textfloatsep}{\savedtextfloatsep}
\setlength{\intextsep}{\savedintextsep}

\section{Additional Derivations and Proofs}
\label{app:derivations_and_proofs}

\subsection{Coefficient Form of Finite Flow Transitions}
\label{sec:app1}

\begin{table}[ht]
    \centering
    \caption{Finite transitions from $t$ to $s$ ($0<s<t<1$) in the form
    of Eq.~\ref{eq:common_transition_form}. The last column is the
    deviation from the coefficient-preserving condition
    $\alpha_{t,s}^{2}+\beta_{t,s}^{2}=s^{2}$. The Flow-GRPO row is the
    standard uncapped Euler--Maruyama construction at interior timesteps;
    endpoint and noise-cap variants are given in Appendix~\ref{sec:app1}.}
    \label{tab:transition_coefficients}
    \vspace{0.5em}
    \begin{tabular}{lccc}
        \toprule
        Transition
        & $\alpha_{t,s}$
        & $\beta_{t,s}$
        & $\alpha_{t,s}^{2}+\beta_{t,s}^{2}-s^{2}$ \\
        \midrule
        ODE Euler
        & $0$
        & $s$
        & $0$ \\
        Flow-GRPO (Euler--Maruyama)
        & $\sigma_{t}\sqrt{\Delta_{t,s}}$
        & $s-\dfrac{\Delta_{t,s}\sigma_{t}^{2}}{2t}$
        & $\dfrac{\Delta_{t,s}^{2}\sigma_{t}^{2}}{t}
          \left(1+\dfrac{\sigma_{t}^{2}}{4t}\right)$ \\
        Flow-CPS
        & $s\sin(\pi\phi/2)$
        & $s\cos(\pi\phi/2)$
        & $0$ \\
        \bottomrule
    \end{tabular}
\end{table}

This section derives the unified coefficient representation of
Eq.~\ref{eq:common_transition_form} for the entries of
Table~\ref{tab:transition_coefficients}, with emphasis on the standard
Euler--Maruyama transition of Flow-GRPO and the modifications used in its
implementation at the schedule endpoint and under a capped noise scale. The
coefficients of Flow-CPS and of the policies studied in this paper are
given directly in the main text. Throughout,
$0<s<t<1$ unless stated otherwise, $\Delta_{t,s}=t-s$, and the estimates
$\hat{\bm{x}}_{0}=\bm{x}_{t}-t\bm{v}_{\theta}$ and
$\hat{\bm{\epsilon}}=\bm{x}_{t}+(1-t)\bm{v}_{\theta}$ are evaluated at
$(\bm{x}_{t},t)$; the arguments of $\bm{v}_{\theta}$ are omitted.

\subsubsection{Deterministic Euler Step}

Inverting the definitions of the estimates gives
$\bm{x}_{t}=(1-t)\hat{\bm{x}}_{0}+t\hat{\bm{\epsilon}}$ and
$\bm{v}_{\theta}=\hat{\bm{\epsilon}}-\hat{\bm{x}}_{0}$. Hence,
\begin{equation}
    \bm{x}_{t}-\Delta_{t,s}\bm{v}_{\theta}
    =
    (1-t+\Delta_{t,s})\hat{\bm{x}}_{0}
    +
    (t-\Delta_{t,s})\hat{\bm{\epsilon}}
    =
    (1-s)\hat{\bm{x}}_{0}
    +
    s\hat{\bm{\epsilon}},
    \label{eq:app_ode_euler_identity}
\end{equation}
which is the identity quoted in Section~\ref{sec:stochastic_transitions}
and the ODE Euler row of Table~\ref{tab:transition_coefficients}. The
identity holds for any velocity output and any span; it rewrites the Euler
update in the
coefficient form and makes no claim about the accuracy of the update as an
ODE solution.

\subsubsection{Flow-GRPO Transition}

Flow-GRPO \citep{liu2025flow} samples from the SDE
\begin{equation}
    \mathrm{d}\bm{x}_{t}
    =
    \left[
        \bm{v}_{\theta}(\bm{x}_{t},t)
        +
        \frac{\sigma_{t}^{2}}{2t}
        \bigl(\bm{x}_{t}+(1-t)\bm{v}_{\theta}(\bm{x}_{t},t)\bigr)
    \right]\mathrm{d}t
    +
    \sigma_{t}\,\mathrm{d}\bm{w}_{t},
    \qquad
    \sigma_{t}=a\sqrt{\frac{t}{1-t}},
    \label{eq:app_flow_grpo_sde}
\end{equation}
integrated from $t=1$ toward $t=0$, where $a>0$ is the noise level. One
Euler--Maruyama step from $t$ to $s$ freezes the coefficients at
$(\bm{x}_{t},t)$ and uses the increment $\mathrm{d}t=-\Delta_{t,s}$:
\begin{equation}
    \bm{x}_{s}^{\mathrm{E}}
    =
    \bm{x}_{t}
    -
    \Delta_{t,s}
    \left[
        \bm{v}_{\theta}
        +
        \frac{\sigma_{t}^{2}}{2t}
        \bigl(\bm{x}_{t}+(1-t)\bm{v}_{\theta}\bigr)
    \right]
    +
    \sigma_{t}\sqrt{\Delta_{t,s}}\,\bm{\xi},
    \qquad
    \bm{\xi}\sim\mathcal{N}(\bm{0},\bm{I}).
    \label{eq:app_flow_grpo_euler_step}
\end{equation}
This is the standard update on which the implementation is based; the
endpoint and noise-cap modifications are described at the end of this
section. The bracketed drift contains
$\bm{x}_{t}+(1-t)\bm{v}_{\theta}=\hat{\bm{\epsilon}}$, and the first two
terms form the deterministic Euler step of
Eq.~\ref{eq:app_ode_euler_identity}. Therefore,
\begin{equation}
    \bm{x}_{s}^{\mathrm{E}}
    =
    (1-s)\hat{\bm{x}}_{0}
    +
    \left(
        s-\frac{\Delta_{t,s}\sigma_{t}^{2}}{2t}
    \right)\hat{\bm{\epsilon}}
    +
    \sigma_{t}\sqrt{\Delta_{t,s}}\,\bm{\xi},
    \label{eq:app_flow_grpo_coefficient_form}
\end{equation}
which gives the Flow-GRPO row of Table~\ref{tab:transition_coefficients},
\begin{equation}
    \alpha_{t,s}^{\mathrm{E}}
    =
    \sigma_{t}\sqrt{\Delta_{t,s}},
    \qquad
    \beta_{t,s}^{\mathrm{E}}
    =
    s-\frac{\Delta_{t,s}\sigma_{t}^{2}}{2t}.
    \label{eq:app_flow_grpo_coefficients}
\end{equation}
The SDE correction leaves the clean-sample coefficient untouched, shrinks the
predicted-noise coefficient by $\Delta_{t,s}\sigma_{t}^{2}/(2t)$, and adds
fresh noise with standard deviation $\sigma_{t}\sqrt{\Delta_{t,s}}$.

Expanding the squared sum of the noise coefficients,
\begin{align}
    \left(\alpha_{t,s}^{\mathrm{E}}\right)^{2}
    +
    \left(\beta_{t,s}^{\mathrm{E}}\right)^{2}
    -
    s^{2}
    &=
    \sigma_{t}^{2}\Delta_{t,s}
    -
    \frac{s\Delta_{t,s}\sigma_{t}^{2}}{t}
    +
    \frac{\Delta_{t,s}^{2}\sigma_{t}^{4}}{4t^{2}}
    \notag\\
    &=
    \frac{\sigma_{t}^{2}\Delta_{t,s}(t-s)}{t}
    +
    \frac{\Delta_{t,s}^{2}\sigma_{t}^{4}}{4t^{2}}
    \notag\\
    &=
    \frac{\Delta_{t,s}^{2}\sigma_{t}^{2}}{t}
    \left(1+\frac{\sigma_{t}^{2}}{4t}\right),
    \label{eq:app_flow_grpo_deviation}
\end{align}
which is the last column of the Flow-GRPO row of
Table~\ref{tab:transition_coefficients}. Both factors are
positive, so the squared sum always exceeds $s^{2}$. Substituting
$\sigma_{t}^{2}=a^{2}t/(1-t)$ gives the equivalent form
\begin{equation}
    \left(\alpha_{t,s}^{\mathrm{E}}\right)^{2}
    +
    \left(\beta_{t,s}^{\mathrm{E}}\right)^{2}
    -
    s^{2}
    =
    \frac{\Delta_{t,s}^{2}a^{2}}{1-t}
    \left(1+\frac{a^{2}}{4(1-t)}\right).
    \label{eq:app_flow_grpo_deviation_noise_level}
\end{equation}

\paragraph{Sign of the predicted-noise coefficient.}
From Eq.~\ref{eq:app_flow_grpo_coefficient_form},
\begin{equation}
    \beta_{t,s}^{\mathrm{E}}<0
    \quad\Longleftrightarrow\quad
    \Delta_{t,s}\sigma_{t}^{2}>2ts
    \quad\Longleftrightarrow\quad
    \Delta_{t,s}a^{2}>2s(1-t).
    \label{eq:app_flow_grpo_negative_beta}
\end{equation}
When this holds, the coefficient of the predicted noise changes sign. The
condition is easier to satisfy for large spans and for source timesteps
close to $1$.

\paragraph{Implementation modifications.}
The implementation of the Flow-GRPO transition differs from
Eq.~\ref{eq:app_flow_grpo_euler_step} only in the noise scale, which is
replaced by an effective scale $\tilde{\sigma}_{t,s}$ that accounts for the
schedule endpoint and for an optional cap on the injected noise. Let
\begin{equation}
    \sigma_{t,s}^{\mathrm{base}}
    =
    \begin{cases}
        a\sqrt{\dfrac{t}{1-t}}, & t<1,\\[1.2ex]
        \dfrac{a}{\sqrt{1-s}}, & t=1,
    \end{cases}
    \qquad
    \tilde{\sigma}_{t,s}
    =
    \min\left\{
        \sigma_{t,s}^{\mathrm{base}},\,
        \frac{c_{\mathrm{clip}}}{\sqrt{\Delta_{t,s}}}
    \right\},
    \label{eq:app_flow_grpo_effective_scale}
\end{equation}
where $c_{\mathrm{clip}}>0$ is the cap on the standard deviation of the
injected noise, and $c_{\mathrm{clip}}=\infty$ when no cap is used. The
implemented update is Eq.~\ref{eq:app_flow_grpo_euler_step} with
$\sigma_{t}$ replaced by $\tilde{\sigma}_{t,s}$, so
Eqs.~\eqref{eq:app_flow_grpo_coefficients}--\eqref{eq:app_flow_grpo_deviation}
hold verbatim with the same substitution, and the injected noise has
standard deviation
$\alpha_{t,s}^{\mathrm{E}}=\tilde{\sigma}_{t,s}\sqrt{\Delta_{t,s}}\leq c_{\mathrm{clip}}$.
Two cases are worth recording explicitly.

\emph{First transition before capping.}
The base scale $a\sqrt{t/(1-t)}$ diverges as $t\rightarrow1$, so the
interior formulas apply to $t<1$ only; at $t_{0}=1$ the implementation
replaces the source timestep by the target timestep inside the denominator
$1-t$, i.e.\ it uses $1-s$ in place of $1-t$, which gives the finite value
$\sigma_{1,s}^{\mathrm{base}}=a/\sqrt{1-s}$ for $s<1$. If the cap
is inactive, substituting $t=1$ and this scale into
Eq.~\ref{eq:app_flow_grpo_coefficients} yields
\begin{equation}
    \alpha_{1,s}^{\mathrm{E}}
    =
    a,
    \qquad
    \beta_{1,s}^{\mathrm{E}}
    =
    s-\frac{a^{2}}{2},
    \qquad
    \left(\alpha_{1,s}^{\mathrm{E}}\right)^{2}
    +
    \left(\beta_{1,s}^{\mathrm{E}}\right)^{2}
    -
    s^{2}
    =
    a^{2}(1-s)+\frac{a^{4}}{4},
    \label{eq:app_flow_grpo_first_step}
\end{equation}
so the first transition injects noise with standard deviation $a$
irrespective of the span. With the cap, the cap threshold at $t=1$ is
$c_{\mathrm{clip}}/\sqrt{1-s}$, which has the same denominator as the base
scale, so $\alpha_{1,s}^{\mathrm{E}}=\min\{a,c_{\mathrm{clip}}\}$.

\emph{Active cap.}
In the region where the cap is active, i.e.\
$(\sigma_{t,s}^{\mathrm{base}})^{2}\Delta_{t,s}>c_{\mathrm{clip}}^{2}$,
the effective scale satisfies
$\tilde{\sigma}_{t,s}^{2}=c_{\mathrm{clip}}^{2}/\Delta_{t,s}$ and
Eqs.~\eqref{eq:app_flow_grpo_coefficients}--\eqref{eq:app_flow_grpo_deviation}
become
\begin{equation}
    \alpha_{t,s}^{\mathrm{E}}
    =
    c_{\mathrm{clip}},
    \qquad
    \beta_{t,s}^{\mathrm{E}}
    =
    s-\frac{c_{\mathrm{clip}}^{2}}{2t},
    \qquad
    \left(\alpha_{t,s}^{\mathrm{E}}\right)^{2}
    +
    \left(\beta_{t,s}^{\mathrm{E}}\right)^{2}
    -
    s^{2}
    =
    \frac{c_{\mathrm{clip}}^{2}\Delta_{t,s}}{t}
    +
    \frac{c_{\mathrm{clip}}^{4}}{4t^{2}}.
    \label{eq:app_flow_grpo_clipped}
\end{equation}
At fixed $t$ and $c_{\mathrm{clip}}$, the deviation is therefore affine in
the span inside the active region, and it remains strictly positive, so
the capped transition still lies outside the coefficient-preserving family.
Outside the active region the uncapped coefficients apply. Whether the cap
is active for a given transition depends on the schedule, the noise level,
and $c_{\mathrm{clip}}$; the values used in our experiments are reported in
Appendix~\ref{app:experimental_details}.

\paragraph{Terminal boundary.}
At $s=0$ the coefficient-preserving condition reduces to
$\alpha_{t,0}=\beta_{t,0}=0$, and relative deviations normalized by $s^{2}$
are undefined; this is consistent with keeping the final transition
deterministic.

\subsection{Conditional KL and Likelihood-Ratio Geometry}
\label{app:policy_geometry}

\subsubsection{KL Equivalence and Fisher Information}

\paragraph{Proof of Proposition~\ref{prop:policy_sensitivity}.}
Fix $(\bm x_t,\bm c,t,s)$ and one admissible coefficient choice.
The old and updated policies have common covariance
$\alpha_{t,s}^2\bm I$, and their mean difference is
\begin{equation}
    \Delta\bm\mu
    =\bm\mu_\beta(\bm v_0+\Delta\bm v)-\bm\mu_\beta(\bm v_0)
    =\kappa_{t,s}\Delta\bm v.
    \label{eq:app_mean_velocity_shift}
\end{equation}
The equal-covariance Gaussian KL formula therefore gives, in either direction,
\begin{equation}
    K_\beta(\Delta\bm v)
    =\frac{\|\Delta\bm\mu\|_2^2}{2\alpha_{t,s}^2}
    =\frac12G_{t,s}^2\|\Delta\bm v\|_2^2.
    \label{eq:app_equal_covariance_kl}
\end{equation}
This identity is exact for finite output differences. Moreover,
$t(1-s)-(1-t)\beta_{t,s}\geq t(1-s)-(1-t)s=t-s>0$,
so $G_{t,s}>0$ throughout the family. For any prescribed nonzero update,
$\|\Delta\bm v\|_2^2/2$ is a positive constant across candidates.
Minimizing the conditional KL in Eq.~\ref{eq:kl_selection_problem}
is therefore equivalent to minimizing $G_{t,s}^2$ or $G_{t,s}$, which
establishes Eq.~\ref{eq:geometry_optimization_problem}.
For two candidates $A,B$, Eq.~\ref{eq:app_equal_covariance_kl} also implies
$G_A\leq G_B$ if and only if
$K_A(\Delta\bm v)\leq K_B(\Delta\bm v)$ for every output update;
the reverse implication follows by evaluating any nonzero update.
The minimizing coefficients are consequently independent of the prescribed
update's nonzero magnitude and direction. \hfill$\square$

\paragraph{Normalized KL response and Fisher information.}
Dividing Eq.~\ref{eq:app_equal_covariance_kl} by
$\|\Delta\bm v\|_2^2/2$ for a nonzero update gives the normalized KL
response $G_{t,s}^2$, interpreted as policy sensitivity in
Eq.~\ref{eq:normalized_policy_sensitivity}.
To obtain the Fisher information, differentiate with the action $\bm x_s$
held fixed:
\begin{equation}
    \nabla_{\bm v}\log q_\beta(\bm x_s\mid\bm v)
    =\frac{\kappa_{t,s}}{\alpha_{t,s}^2}
       (\bm x_s-\bm\mu_\beta(\bm v)).
    \label{eq:app_velocity_score}
\end{equation}
Under the policy at $\bm v$, write
$\bm x_s-\bm\mu_\beta(\bm v)=\alpha_{t,s}\bm\xi$ with
$\bm\xi\sim\mathcal N(\bm0,\bm I)$. Hence
\begin{equation}
    \mathcal I_{\bm v}
      =\mathbb E\!\left[(\nabla_{\bm v}\log q_\beta)
                        (\nabla_{\bm v}\log q_\beta)^\top\right]
      =\frac{\kappa_{t,s}^2}{\alpha_{t,s}^2}\bm I
      =G_{t,s}^2\bm I.
    \label{eq:app_velocity_fisher}
\end{equation}

\subsubsection{Likelihood-Ratio Moments}
\label{app:ratio_moments}

Keep the state and timestep pair fixed, and prescribe an output update
$\Delta\bm v\in\mathbb R^D$ independently of the action used to evaluate
the ratio. For an action from the old policy, write
$\bm x_s=\bm\mu_\beta(\bm v_0)+\alpha_{t,s}\bm\xi$,
$\bm\xi\sim\mathcal N(\bm0,\bm I)$, and set
$\bm z=\kappa_{t,s}\Delta\bm v/\alpha_{t,s}$.
For $r=q_\beta(\bm x_s\mid\bm v_0+\Delta\bm v)/
q_\beta(\bm x_s\mid\bm v_0)$, the Gaussian log-density difference is
\begin{equation}
    \begin{aligned}
    \log r
    &=-\frac{\|\alpha_{t,s}\bm\xi-\kappa_{t,s}\Delta\bm v\|_2^2
                  -\|\alpha_{t,s}\bm\xi\|_2^2}{2\alpha_{t,s}^2}\\
    &=\bm\xi^\top\bm z-\frac12\|\bm z\|_2^2,
    \qquad \|\bm z\|_2=G_{t,s}\|\Delta\bm v\|_2=d.
    \end{aligned}
    \label{eq:app_log_ratio_expansion}
\end{equation}
Since $\bm\xi^\top\bm z\sim\mathcal N(0,d^2)$,
\begin{equation}
    \log r\sim\mathcal N(-d^2/2,d^2),
    \qquad \operatorname{Var}(\log r)=2K_\beta(\Delta\bm v).
    \label{eq:log_ratio_distribution}
\end{equation}
In particular,
$-\mathbb E_{q_{\mathrm{old}}}[\log r]=d^2/2
=D_{\mathrm{KL}}(q_{\mathrm{old}}\Vert q_\theta)
=K_\beta(\Delta\bm v)$, where the last equality uses the common
covariance. The log-normal moments are
\begin{equation}
    \mathbb E_{q_{\mathrm{old}}}[r^m]
       =\exp\!\left(\frac{m^2-m}{2}d^2\right),
    \qquad
    \mathbb E[r]=1,\qquad \operatorname{Var}(r)=e^{d^2}-1.
    \label{eq:app_ratio_moment}
\end{equation}
These conditional moments characterize likelihood changes without
assumptions on rewards or advantages.

\subsubsection{Dimension-Averaged Log-Densities}
\label{app:dimension_averaging}

The implemented ratio is $\bar r=\exp((\log q_\theta-
\log q_{\mathrm{old}})/D)=r^{1/D}$, where $D\geq1$ is the latent
dimension. Under the same fixed-state, prescribed-update comparison,
Eq.~\ref{eq:log_ratio_distribution} gives
\begin{equation}
    \log\bar r\sim\mathcal N\!\left(-\frac{d^2}{2D},\frac{d^2}{D^2}\right),
    \qquad
    \overline K_\beta=\frac{K_\beta}{D}
        =\frac{G_{t,s}^2}{2D}\|\Delta\bm v\|_2^2.
    \label{eq:app_averaged_log_ratio}
\end{equation}
Consequently, $\operatorname{Var}(\log\bar r)=2\overline K_\beta/D$.
With $D$ fixed across candidate coefficients, minimizing $G_{t,s}$
minimizes both the normalized KL and the averaged log-ratio variance for
the same output update.

The quantity $\bar r$ is a scaled surrogate, not generally the density
ratio of the two policies. Its moments are
\begin{equation}
    \mathbb E[\bar r]
       =\exp\!\left(-\frac{d^2}{2D}+\frac{d^2}{2D^2}\right),
    \qquad
    \operatorname{Var}(\bar r)
       =\exp\!\left(-\frac{d^2}{D}+\frac{d^2}{D^2}\right)
          \left(\exp\!\left(\frac{d^2}{D^2}\right)-1\right).
    \label{eq:app_averaged_ratio_moments}
\end{equation}
For $D>2$ this variance is not globally monotone in $d^2$ and tends to
zero as $d^2\to\infty$. The KL and log-ratio variance ordering therefore
does not imply a global ordering of $\operatorname{Var}(\bar r)$.

\subsubsection{Clipping-Window Exceedance}
\label{app:clipping_window}

\begin{proposition}[Population clipping-window exceedance]
\label{prop:clipping_window}
Fix the state, timestep pair, and dimension $D\geq1$, and prescribe
$\Delta\bm v$ independently of the evaluation draw from each candidate's
old policy. Let $0<\epsilon<1$, $d=G_{t,s}\|\Delta\bm v\|_2$, and
\begin{equation}
    a=-D\log(1-\epsilon),\qquad b=D\log(1+\epsilon).
    \label{eq:app_clipping_thresholds}
\end{equation}
For $d>0$, the probability of leaving the clipping window is
\begin{equation}
    \begin{aligned}
    p_{\mathrm{out}}(d)
      &:=\Pr\!\left(\bar r\notin[1-\epsilon,1+\epsilon]\right)\\
      &=\Phi\!\left(\frac d2-\frac ad\right)
        +\Phi\!\left(-\frac d2-\frac bd\right),
    \end{aligned}
    \label{eq:app_clipping_exceedance}
\end{equation}
where $\Phi$ is the standard normal distribution function, and
$p_{\mathrm{out}}(0)=0$. If $a\leq3b$, then $p_{\mathrm{out}}(d)$ is
strictly increasing for $d>0$. Hence, under this condition, minimizing
$G_{t,s}$ also minimizes the population clipping-window exceedance
probability for every prescribed nonzero output update.
\end{proposition}

\paragraph{Proof.}
By Eq.~\ref{eq:app_averaged_log_ratio},
$\log\bar r=(-d^2/2+dZ)/D$ with $Z\sim\mathcal N(0,1)$.
The lower-window event is $Z<d/2-a/d$, and the upper-window event is
$Z>d/2+b/d$, which gives Eq.~\ref{eq:app_clipping_exceedance}.
At $d=0$, $\bar r=1$ almost surely.

Let $\varphi$ denote the standard normal density. Differentiating for $d>0$
gives
\begin{equation}
    p_{\mathrm{out}}'(d)
      =\varphi\!\left(\frac d2-\frac ad\right)
         \left(\frac12+\frac{a}{d^2}\right)
       +\varphi\!\left(\frac d2+\frac bd\right)
         \left(-\frac12+\frac{b}{d^2}\right).
    \label{eq:app_clipping_derivative}
\end{equation}
If $d^2\leq2b$, the first term is positive and the second is nonnegative.
If $d^2>2b$, use $a>b>0$ and $a\leq3b$ to obtain
\begin{equation}
    \frac{\varphi(d/2-a/d)}{\varphi(d/2+b/d)}
      =\exp\!\left(\frac{a+b}{2}-\frac{a^2-b^2}{2d^2}\right)
      \geq\exp\!\left(\frac{(a+b)(3b-a)}{4b}\right)
      \geq1.
    \label{eq:app_clipping_density_comparison}
\end{equation}
Substitution into Eq.~\ref{eq:app_clipping_derivative} then yields
$p_{\mathrm{out}}'(d)\geq
\varphi(d/2+b/d)(a+b)/d^2>0$.
This proves strict monotonicity. \hfill$\square$

The condition $a\leq3b$ is equivalent to
$-\log(1-\epsilon)\leq3\log(1+\epsilon)$ and holds for the
experimental choice $\epsilon=0.001$, independently of $D$.
The proposition orders the total population probability of leaving the
window. Active PPO clipping additionally depends on the advantage sign:
the clipped branch is active above the upper threshold for positive
advantages and below the lower threshold for negative advantages.
The result also does not establish an ordering of empirical minibatch
fractions when the output update depends on the same samples used to
evaluate those fractions.

\subsubsection{Fisher Metric in Parameter Space}
\label{app:fisher_geometry}

Let $\bm J_\theta=\partial\bm v_\theta/\partial\theta$ at the fixed
conditioning state. The chain rule gives
\begin{equation}
    \nabla_\theta\log q_\theta
        =\bm J_\theta^\top\nabla_{\bm v}\log q_\theta,
    \qquad
    \mathcal I_\theta
        =\bm J_\theta^\top\mathcal I_{\bm v}\bm J_\theta
        =G_{t,s}^2\bm J_\theta^\top\bm J_\theta.
    \label{eq:app_parameter_fisher}
\end{equation}
For the same model and state, coefficient selection changes the scalar
factor multiplying this conditional metric. It does not change the
Jacobian or its directional anisotropy. Unlike the exact KL identity in
output space, the corresponding quadratic approximation in parameter
space uses a local linearization of the model. A shared parameter update
need not independently realize prescribed output changes across states.

\subsubsection{Equivalent KL-Budget Interpretation}

The same conditional comparison can be expressed as an output-space
allowance. For a common budget $\delta>0$, define
$\mathcal T_\beta(\delta)=\{\Delta\bm v\in\mathbb R^D:
K_\beta(\Delta\bm v)\leq\delta\}$.
Equation~\ref{eq:app_equal_covariance_kl} gives
\begin{equation}
    K_\beta(\Delta\bm v)\leq\delta
    \quad\Longleftrightarrow\quad
    \|\Delta\bm v\|_2\leq\frac{\sqrt{2\delta}}{G_{t,s}}.
    \label{eq:app_budget_ball}
\end{equation}
Thus the admissible region is an entire centered ball with radius
\begin{equation}
    \rho_\beta(\delta)=\frac{\sqrt{2\delta}}{G_{t,s}(\beta_{t,s})}.
    \label{eq:velocity_update_radius}
\end{equation}
Comparing these concentric balls yields
\begin{equation}
    G_A\leq G_B
    \quad\Longleftrightarrow\quad
    \mathcal T_B(\delta)\subseteq\mathcal T_A(\delta)
       \quad\text{for every }\delta>0.
    \label{eq:app_budget_ordering}
\end{equation}
The same equivalence holds at any one positive budget. Minimizing
$G_{t,s}$ therefore also maximizes the radius for every positive budget.
For a normalized budget $\overline K_\beta\leq\bar\delta$, the radius
is $\sqrt{2D\bar\delta}/G_{t,s}$, so the optimum is unchanged at fixed
$D$. At zero budget, every candidate permits only the zero update.

\subsection{Geometry-Optimal Transition}
\label{app:optimal_transition}

\subsubsection{Proof of Theorem~\ref{thm:optimal_transition}}

Write $A=t(1-s)>0$ and $B=1-t\geq0$. For $0\leq\beta<s$,
$A-B\beta\geq A-Bs=t-s>0$, so
\begin{equation}
    G_{t,s}(\beta)=\frac{A-B\beta}{\sqrt{s^2-\beta^2}},
    \qquad
    \frac{\mathrm dG_{t,s}}{\mathrm d\beta}
       =\frac{A\beta-Bs^2}{(s^2-\beta^2)^{3/2}}.
    \label{eq:app_geometry_derivative}
\end{equation}
For $t<1$, the derivative changes sign from negative to positive at
$\beta^\star=Bs^2/A$. This point lies in the domain because
$\beta^\star/s=s(1-t)/(t(1-s))<1$, which is equivalent to $s<t$.
For $t=1$, $B=0$ and the derivative is strictly positive for $\beta>0$,
so the unique minimum is the boundary point $\beta^\star=0$.
In both cases $\alpha^\star=\sqrt{s^2-(\beta^\star)^2}$ is uniquely
determined by positivity and the coefficient constraint. Extending the
domain to $-s<\beta<s$ leaves the unique minimizer unchanged:
$A\beta-Bs^2<0$ throughout $\beta<0$.

Substitution yields
\begin{equation}
    (G_{t,s}^{\star})^2
      =\frac{A^2-B^2s^2}{s^2}
      =\frac{(A-Bs)(A+Bs)}{s^2}
      =\frac{(t-s)(t+s-2ts)}{s^2}.
    \label{eq:app_optimal_geometry_squared}
\end{equation}
Taking the positive square root proves the stated minimum.
Proposition~\ref{prop:policy_sensitivity} then establishes that these
coefficients minimize the conditional KL for every prescribed nonzero
velocity-output update. \hfill$\square$

\subsubsection{Span Dependence and Boundaries}

The normalized stochasticity parameter satisfies
$\cos(\pi\phi_{t,s}^{\star}/2)=\beta_{t,s}^{\star}/s=s(1-t)/(t(1-s))$.
At fixed $s\in(0,1)$, the derivative of this retained-noise fraction
with respect to $t$ is
$-s/(t^2(1-s))<0$. The retained noise coefficient fraction therefore
decreases as the span grows in this fixed-target comparison.
For $t=s+\Delta$,
\begin{equation}
    (G_{s+\Delta,s}^{\star})^2
      =\frac{2(1-s)}{s}\Delta+\frac{1-2s}{s^2}\Delta^2,
    \qquad
    G_{s+\Delta,s}^{\star}
      =\sqrt{\frac{2(1-s)}{s}}\sqrt\Delta+O(\Delta^{3/2}).
    \label{eq:app_small_step_geometry}
\end{equation}
The expansion holds as $\Delta\downarrow0$ with $s$ fixed in $(0,1)$.
For the full admissible range,
\begin{equation}
    \frac{\partial}{\partial t}(G_{t,s}^{\star})^2
      =\frac{2[t(1-2s)+s^2]}{s^2}>0.
    \label{eq:app_geometry_span_derivative}
\end{equation}
Indeed, if $s\leq1/2$, the first term in brackets is nonnegative and
$s^2>0$. If $s>1/2$, the bracket is at least
$1-2s+s^2=(1-s)^2>0$, since $t\leq1$. This proves strict monotonicity.
At $t=1$,
\begin{equation}
    \beta_{1,s}^{\star}=0,\qquad \alpha_{1,s}^{\star}=s,
    \qquad G_{1,s}^{\star}=\frac{1-s}{s}.
    \label{eq:app_endpoint_geometry}
\end{equation}
The terminal case $s=0$ is outside the stochastic family: the coefficient
constraint forces $\alpha_{t,0}=\beta_{t,0}=0$, giving the deterministic
Euler update and no non-degenerate Gaussian likelihood.

\subsubsection{Correspondence with SDE-DPM-Solver++}
\label{app:dpm_correspondence}

For $0<s<t<1$, the one-step SDE-DPM-Solver++ transition with data prediction
\citep{lu2022dpmsolverplusplus} is
\begin{equation}
    \bm x_s
      =\frac{\sigma_s}{\sigma_t}e^{-h}\bm x_t
        +a_s(1-e^{-2h})\hat{\bm x}_0
        +\sigma_s\sqrt{1-e^{-2h}}\bm\xi,
    \qquad h=\log\frac{a_s/\sigma_s}{a_t/\sigma_t}.
    \label{eq:app_dpm_solver_step}
\end{equation}
Here $a_u$ and $\sigma_u$ denote the clean-sample and noise coefficients
of the marginal schedule, distinct from the transition coefficients
$\alpha_{t,s},\beta_{t,s}$. For rectified flow, $a_u=1-u$ and
$\sigma_u=u$. Setting $b=e^{-h}=s(1-t)/(t(1-s))$ and using
$\bm x_t=(1-t)\hat{\bm x}_0+t\hat{\bm\epsilon}$ gives
\begin{equation}
    \begin{aligned}
    \frac{s}{t}b\bm x_t+(1-s)(1-b^2)\hat{\bm x}_0
       &=(1-s)\hat{\bm x}_0+sb\hat{\bm\epsilon},\\
    sb&=\beta_{t,s}^{\star},
       \qquad s\sqrt{1-b^2}=\alpha_{t,s}^{\star}.
    \end{aligned}
    \label{eq:app_solver_coefficient_correspondence}
\end{equation}
Thus the coefficients agree exactly for every velocity prediction, without
an oracle assumption. The correspondence extends continuously to $t=1$
through $b=0$; the log-SNR itself is not finite at that endpoint.

\paragraph{A Gaussian Markov interpretation.}
A related interpretation uses a Gaussian Markov process with the same
one-time conditional marginals as the linear probability path. For
$0<s<t\leq1$, define
\begin{equation}
    \begin{aligned}
    \bm x_s\mid\bm x_0
       &\sim\mathcal N((1-s)\bm x_0,s^2\bm I),\\
    \bm x_t\mid\bm x_s
       &\sim\mathcal N(\gamma_{t,s}\bm x_s,
                       (t^2-\gamma_{t,s}^2s^2)\bm I),
       \qquad \gamma_{t,s}=\frac{1-t}{1-s}.
    \end{aligned}
    \label{eq:app_markov_coupling}
\end{equation}
The innovation variance is positive because $t>\gamma_{t,s}s$. The
resulting marginal is
$\bm x_t\mid\bm x_0\sim\mathcal N((1-t)\bm x_0,t^2\bm I)$.
This Markov coupling differs from interpolating with a single shared
noise vector at all times; the marginal path alone does not specify the
joint law across time.

Gaussian conditioning gives
\begin{equation}
    \begin{aligned}
    \mathbb E[\bm x_s\mid\bm x_t,\bm x_0]
      &=(1-s)\bm x_0+\frac{s^2\gamma_{t,s}}{t^2}
                          (\bm x_t-(1-t)\bm x_0),\\
    \operatorname{Cov}(\bm x_s\mid\bm x_t,\bm x_0)
      &=\left(s^2-\frac{s^4\gamma_{t,s}^2}{t^2}\right)\bm I.
    \end{aligned}
    \label{eq:app_posterior_moments}
\end{equation}
Replacing $\bm x_0$ by $\hat{\bm x}_0$ and substituting
$\bm x_t-(1-t)\hat{\bm x}_0=t\hat{\bm\epsilon}$ yields
$\beta_{t,s}=s^2\gamma_{t,s}/t=\beta_{t,s}^{\star}$ and
$\alpha_{t,s}=\alpha_{t,s}^{\star}$. This is a plug-in posterior
interpretation of the selected transition, rather than a claim that the
learned reconstruction gives the true posterior of the data.

\subsection{Derivation and Scope of Geometry-Matched Scaling}
\label{app:geometry_matched_scaling}

The weights match two explicit scales of the selected conditional policy.
For actions sampled from the current policy,
Eq.~\ref{eq:app_velocity_score} becomes
$\nabla_{\bm v}\log q_\theta
=\operatorname{sgn}(\kappa_{t,s}^{\star})G_{t,s}^{\star}\bm\xi$,
where $\bm\xi\sim\mathcal N(\bm0,\bm I)$; the action is held fixed
when differentiating. The same identity applies at
$\theta=\theta_{\mathrm{old}}$ for recorded rollout actions. Multiplying
the score by $w^{\mathrm{policy}}=c_{\mathrm{geo}}/G_{t,s}^{\star}$
makes this explicit coefficient have magnitude $c_{\mathrm{geo}}$.

For a stored old-policy action evaluated after the model changes, the
standardized residual is instead
\begin{equation}
    \bm\xi_\theta
      =\bm\xi_{\mathrm{old}}
         -\frac{\kappa_{t,s}^{\star}}{\alpha_{t,s}^{\star}}\Delta\bm v,
    \qquad
    \nabla_{\bm v}\log q_\theta
      =\frac{\kappa_{t,s}^{\star}}{\alpha_{t,s}^{\star}}
           \bm\xi_{\mathrm{old}}-(G_{t,s}^{\star})^2\Delta\bm v.
    \label{eq:app_old_action_score}
\end{equation}
Here $\bm\xi_{\mathrm{old}}$ is standard Gaussian under the old policy.
Consequently, canceling the leading coefficient does not remove all
geometry dependence of later optimization gradients. Advantages, the
likelihood ratio, clipping, the model Jacobian, and the optimizer further
affect the parameter update.

For the reference term, both policies use the selected coefficients, so
\begin{equation}
    D_{\mathrm{KL}}(q_\theta\Vert q_{\mathrm{ref}})
       =\frac12(G_{t,s}^{\star})^2
          \|\bm v_\theta-\bm v_{\mathrm{ref}}\|_2^2,
    \qquad
    \nabla_{\bm v_\theta}^2 D_{\mathrm{KL}}
       =(G_{t,s}^{\star})^2\bm I.
    \label{eq:app_reference_kl_curvature}
\end{equation}
The weight $w^{\mathrm{KL}}=(c_{\mathrm{geo}}/G_{t,s}^{\star})^2$
maps this curvature to $c_{\mathrm{geo}}^2\bm I$. The same covariance
also makes the KL symmetric, although its reference-regularization role
differs from the updated-to-old conditional KL used for policy selection.

Dimension averaging introduces a factor $1/D$ into both the log-density
score and the reference KL. With the same weights, their matched explicit
scales become $c_{\mathrm{geo}}/D$ and $c_{\mathrm{geo}}^2/D$,
respectively. Thus the relative transition-wise weights are unchanged
at fixed $D$. This does not turn the averaged log-density into the log of
a normalized joint policy, or make its score outer product equal to the
curvature of the averaged KL.

\paragraph{Estimator scale and signal-to-noise ratio.}
At a fixed state, timestep pair, and latent dimension, with
$\theta=\theta_{\mathrm{old}}$, an unclipped negative policy-surrogate
gradient has the form $\widehat{\bm g}_\beta=(G_\beta/D)\bm Y_\beta$,
where $\bm Y_\beta=A_\beta\bm J^\top\bm\xi$ and
$\bm J=\partial\bm v/\partial\theta$.
For nonzero variance, the explicit factor cancels from
$\|\mathbb E\widehat{\bm g}_\beta\|_2/
\sqrt{\operatorname{tr}\operatorname{Cov}(\widehat{\bm g}_\beta)}$.
The action--advantage distribution nevertheless depends on the policy,
so this cancellation establishes neither equality nor an ordering of
signal-to-noise ratios across coefficient choices.
For example, in one dimension with $x=\mu+\alpha\xi$, reward $R(x)=x$,
and its exact mean as baseline, $A=\alpha\xi$ and the reward-gradient
estimator is $\kappa\xi^2$; its signal-to-noise ratio is $1/\sqrt2$
for every $\alpha>0$. Thus increasing noise need not increase estimator
signal-to-noise ratio.

\paragraph{The global matching multiplier.}
Writing $P=\mathbb E[\bar\ell^{\mathrm{clip}}/G^\star]$ and
$K=\mathbb E[\overline D_{\mathrm{KL}}/(G^\star)^2]$ gives
$\mathcal L(c,\lambda)=cP+\lambda c^2K$ and
$\mathcal L(ac,\lambda/a)=a\mathcal L(c,\lambda)$ for $a>0$.
Thus $c$ affects the relative reference penalty and an overall loss scale.
Plain SGD offsets the latter by dividing the learning rate by $a$.
AdamW instead cancels uniform gradient scaling at the same learning rate
in the zero-$\epsilon$, unclipped limit with zero-initialized moments.
Finite optimizer $\epsilon$ and a fixed gradient-clipping threshold can
prevent exact equivalence. The multiplier does not change the sampling
policy or its conditional KL cost for a prescribed velocity update.

\input{sec/experimental_details}
\input{sec/x_appendix}

\end{document}

%% file: math_commands.tex
\usepackage{amsmath,amsfonts,bm}

\def\eqref#1{equation~\ref{#1}}

\def\1{\bm{1}}

\DeclareMathAlphabet{\mathsfit}{\encodingdefault}{\sfdefault}{m}{sl}
\SetMathAlphabet{\mathsfit}{bold}{\encodingdefault}{\sfdefault}{bx}{n}



%% file: sec/4_exp.tex
\section{Experiments}
We evaluate four-step image and video generation with Geo-FSRL, Flow-GRPO \citep{liu2025flow}, Flow-CPS \citep{wang2025coefficients}, and a four-step DiffusionNFT adaptation \citep{zheng2026diffusionnft}. GRPO-based comparisons share single-stochastic-transition rollouts. Flow-CPS and Geo-FSRL use gradient-scale matching with weights from their respective coefficients unless ablated. Both tasks use 4 sampling steps in training and evaluation. For each task, all reward curves and tabulated automatic scores use the same fixed validation set of 256 prompts, disjoint from training prompts. The horizontal axes count optimizer updates. Training uses 32 NVIDIA H100 GPUs; hyperparameters, schedules, and reward normalization appear in Appendix~\ref{app:experimental_details}.

\subsection{Few-Step Text-to-Image}
\textbf{Experimental Setup}. We fine-tune all diffusion-transformer parameters of Qwen-Image-2512-Lightning-4steps\footnote{Distilled checkpoint: \url{https://huggingface.co/lightx2v/Qwen-Image-Lightning}.} \citep{wu2025qwen, lightx2v}, distilled from the 20B MMDiT Qwen-Image, at 720 $\times$ 720 resolution.
\par
\textbf{Reward Models}. We optimize HPSv3 \citep{ma2025hpsv3} for human preference, alone or jointly with CLIPScore \citep{radford2021learning} for image-text alignment.
\par
\textbf{Dataset}. Prompts come from the text-to-image generation split of OpenGPT-4o-Image \citep{chen2025opengpt}, a diverse corpus of about 40k prompts.
\par
\textbf{Single-reward Training Results}. We compare Geo-FSRL with Flow-CPS at normalized stochasticity $\phi\in\{0.3,0.5,0.7,1.0\}$, using each policy's matched weights. Geo-FSRL derives coefficients from each timestep pair. It improves similarly to the best tested Flow-CPS setting, $\phi=0.7$, in early and middle training (Figure~\ref{fig:exp1}(a)), but finishes higher: 14.05 versus 13.90; Flow-GRPO reaches 13.62 (Table~\ref{tab:exp1}). Maximum injected noise, $\phi=1$, is included and differs from the geometry optimum at internal timesteps. The tested DiffusionNFT adaptation exhibits early reward deterioration and visible pixel noise; this observation is specific to that adaptation.
\par
A global Flow-CPS $\phi$ specifies one geometry factor per stochastic
transition. Changing $\phi$ can improve the factor at one transition
while worsening it at another, whereas the closed-form schedule
minimizes each separately. The scan therefore compares globally fixed
noise splits with a transition-wise rule, rather than testing a monotonic
relation between any single geometry factor and final reward.
\par
\textbf{Multi-reward Joint Training Results}. Joint HPSv3 and CLIPScore training retains the best single-reward Flow-CPS setting, $\phi=0.7$. Geo-FSRL leads both final metrics (Figure~\ref{fig:exp1}(b,c), Table~\ref{tab:exp1}): 13.80 HPSv3 versus Flow-CPS's 13.58, and 1.008 CLIPScore versus Flow-GRPO's 0.996, absolute gains of 0.22 and 0.012 over the respective strongest baselines. Qualitative examples appear in Appendix~\ref{sec:appvis}.
\par
\textbf{Human evaluation.} Five evaluators blinded to method identity compare Geo-FSRL and Flow-CPS on 20 image pairs from randomly sampled prompts, with shared prompts and seeds and shuffled presentation order. Geo-FSRL receives 54\% of Visual Quality votes (aesthetics) and 58\% of Overall Preference votes (prompt adherence and visual appeal; Table~\ref{tab:exphuman}). These percentages are descriptive results from a small evaluation.

\begin{figure}[!htb]
\centering
\setlength{\abovecaptionskip}{5pt}
\includegraphics[width=\linewidth]{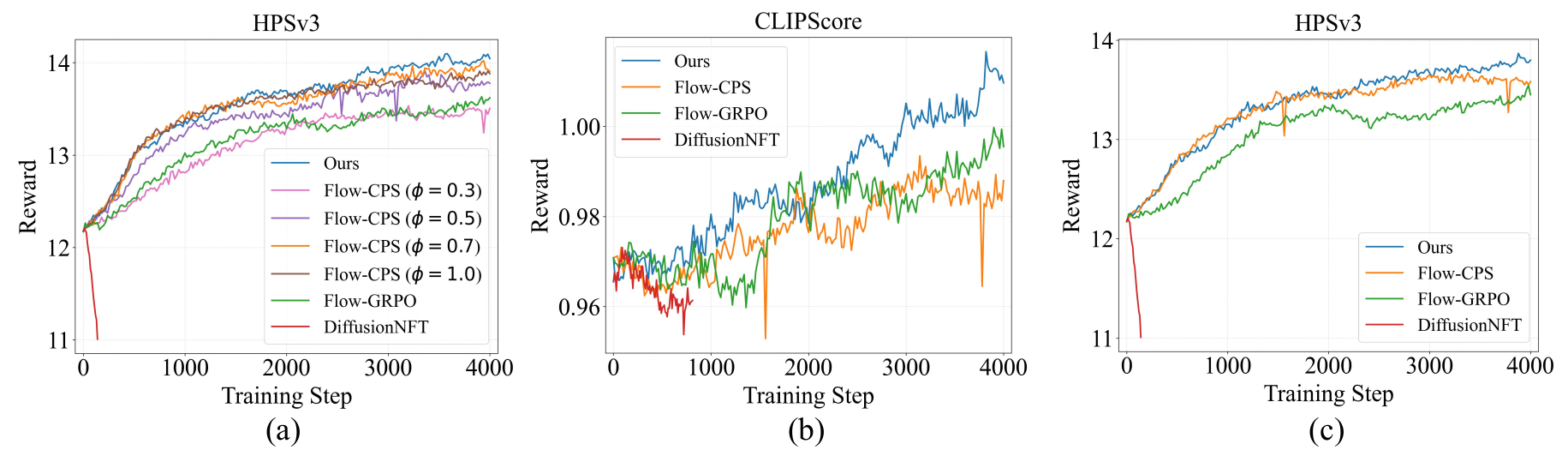}
\caption{Four-step image validation rewards versus optimizer updates: (a) HPSv3-only; (b,c) joint HPSv3 and CLIPScore. Flow-CPS curves use matched weights for their respective $\phi$ values. DiffusionNFT denotes the tested four-step adaptation.}
\label{fig:exp1}
\end{figure}

\begin{table}[!htb]
    \centering
    \setlength{\abovecaptionskip}{3pt}
    \setlength{\belowcaptionskip}{3pt}
    \caption{Final validation reward scores for single-reward setting and multi-reward setting text-to-image training.}
    \label{tab:exp1}
    \begin{tabular}{lccc}
        \toprule
        & \multicolumn{1}{c}{Single-Reward}
        & \multicolumn{2}{c}{Multi-Reward} \\
        \cmidrule(lr){2-2}
        \cmidrule(lr){3-4}
        Method
        & HPSv3
        & CLIPScore
        & HPSv3 \\
        \midrule
        Qwen-Image-Lightning & 12.17 & 0.971 & 12.17 \\
        Flow-GRPO \citep{liu2025flow} & 13.62 {\scriptsize\textcolor{gray}{(+1.45)}} & 0.996 {\scriptsize\textcolor{gray}{(+0.025)}} & 13.45 {\scriptsize\textcolor{gray}{(+1.28)}} \\
        Flow-CPS \citep{wang2025coefficients} & 13.90 {\scriptsize\textcolor{gray}{(+1.73)}} & 0.988 {\scriptsize\textcolor{gray}{(+0.017)}} & 13.58 {\scriptsize\textcolor{gray}{(+1.41)}} \\
        Geo-FSRL (Ours) & \textbf{14.05} {\scriptsize\textcolor{gray}{(+1.88)}} & \textbf{1.008} {\scriptsize\textcolor{gray}{(+0.037)}} & \textbf{13.80} {\scriptsize\textcolor{gray}{(+1.63)}} \\
        \bottomrule
    \end{tabular}
\end{table}
\vspace{1em}

\begin{table}[!tb]
    \centering
    \setlength{\abovecaptionskip}{3pt}
    \setlength{\belowcaptionskip}{3pt}
    \caption{Human evaluation of 20 image pairs by five evaluators blinded to method identity. We report the percentage of preference votes for each method.}
    \label{tab:exphuman}
    \begin{tabular}{lcc}
        \toprule
        & Visual Quality (\%)
        & Overall Preference (\%) \\
        \midrule
        Flow-CPS & 46 & 42 \\
        Geo-FSRL (Ours) & \textbf{54} & \textbf{58}  \\
        \bottomrule
    \end{tabular}
\end{table}

\textbf{Policy Reweighting Ablation.}
We vary policy weighting while retaining matched KL weighting for both Flow-CPS and our Geo-FSRL, ensuring consistent KL regularization strength across transitions. As shown in Table~\ref{tab:policy_reweight_ablation}, policy
weighting improves Flow-CPS by 0.21 and Geo-FSRL by 0.44. Geo-FSRL scores
0.08 below Flow-CPS without policy weighting and 0.15 above with their
respective matched weights. The observed ordering therefore changes with
policy weighting. These results support the combined design in this
setting, evaluating coefficient selection together with its corresponding
weights. Appendix~\ref{sec:ablations} provides the validation curves and
a controlled gradient-scale analysis.

\begin{table}[!tb]
    \centering
    \setlength{\abovecaptionskip}{3pt}
    \setlength{\belowcaptionskip}{3pt}
    \caption{Single-reward policy-weighting ablation on the validation set. Each method uses its own matched weights; KL weighting remains enabled throughout. Gain is the change from enabling policy weighting.}
    \label{tab:policy_reweight_ablation}
    \begin{tabular}{lccc}
        \toprule
        Method
        & \shortstack{Without policy\\reweighting}
        & \shortstack{With policy\\reweighting}
        & Gain \\
        \midrule
        Flow-CPS & \textbf{13.69} & 13.90 & +0.21 \\
        Geo-FSRL (Ours) & 13.61 & \textbf{14.05} & \textbf{+0.44} \\
        \bottomrule
    \end{tabular}
\end{table}

\subsection{Few-Step Text-to-Video}
\textbf{Experimental Setup}. We distill Wan2.1-T2V-1.3B \citep{wan2025wan} to four steps using DMD2 with the LightX2V training framework\footnote{\url{https://github.com/ModelTC/LightX2V/tree/main/lightx2v_train}} \citep{lightx2v}. We then fine-tune all diffusion-transformer parameters, generating 81 frames at 832 $\times$ 480 resolution.
\par
\textbf{Reward Models}. We jointly optimize HPSv3-General and HPSv3-Percentile for frame-level visual quality and caption alignment, and VideoAlign-MQ for motion quality \citep{liu2026improving,meituanlongcatteam2025longcatvideotechnicalreport}; Appendix~\ref{app:experimental_details} specifies aggregation.
\par
\textbf{Dataset}. Following Self-Forcing \citep{huang2026self}, we use prompts from LLM-extended VidProM \citep{wang2024vidprom}.
\par
\textbf{Results}. Geo-FSRL leads all three final metrics (Figure~\ref{fig:exp3}, Table~\ref{tab:exp2}), gaining 1.09 HPSv3-General, 0.80 HPSv3-Percentile, and 0.81 VideoAlign-MQ over the respective strongest baselines. Validation curves show a clearer advantage than for single-reward images. These results evaluate the complete method on a second modality; component comparisons remain in the image ablation. Selected video frames appear in Appendix~\ref{sec:appvis}.

\begin{figure}[!htb]
\centering
\setlength{\abovecaptionskip}{5pt}
\includegraphics[width=\linewidth]{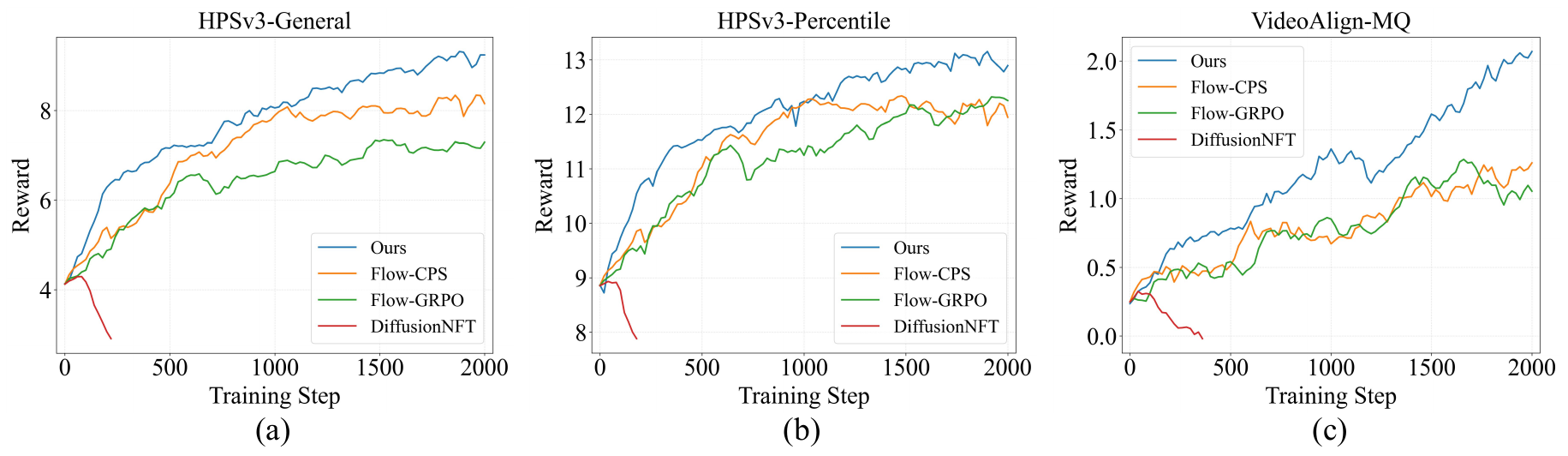}
\caption{Four-step video validation rewards versus optimizer updates under joint optimization. Flow-CPS uses matched weights; DiffusionNFT denotes the tested four-step adaptation.}
\label{fig:exp3}
\end{figure}

\begin{table}[!htb]
    \centering
    \setlength{\abovecaptionskip}{3pt}
    \setlength{\belowcaptionskip}{3pt}
    \caption{Final validation reward scores for joint text-to-video training.}
    \label{tab:exp2}
    \begin{tabular}{lccc}
        \toprule
        Method
        & HPSv3-General
        & HPSv3-Percentile
        & VideoAlign-MQ \\
        \midrule
        Wan2.1-T2V-1.3B & 4.13 & 8.85 & 0.25 \\
        Flow-GRPO \citep{liu2025flow} & 7.29 {\scriptsize\textcolor{gray}{(+3.16)}} & 12.10 {\scriptsize\textcolor{gray}{(+3.25)}} & 1.05 {\scriptsize\textcolor{gray}{(+0.80)}} \\
        Flow-CPS \citep{wang2025coefficients} & 8.15 {\scriptsize\textcolor{gray}{(+4.02)}} & 11.94 {\scriptsize\textcolor{gray}{(+3.09)}} & 1.26 {\scriptsize\textcolor{gray}{(+1.01)}} \\
        Geo-FSRL (Ours) & \textbf{9.24} {\scriptsize\textcolor{gray}{(+5.11)}} & \textbf{12.90} {\scriptsize\textcolor{gray}{(+4.05)}} & \textbf{2.07} {\scriptsize\textcolor{gray}{(+1.82)}} \\
        \bottomrule
    \end{tabular}
\end{table}

%% file: sec/5_related.tex
\section{Related Work}
\textbf{Policy optimization for fast generation.}
DDPO and DPOK use policy gradients for stochastic denoising \citep{black2024training, fan2023dpok}.
Their reference implementations use DDIM ($\eta=1$) \citep{song2021denoising}, yielding ancestral posterior-variance transitions on the timestep grid; their formulations do not require this choice.\footnote{Default settings: \href{https://github.com/jannerm/ddpo/blob/main/config/base.py}{DDPO} and \href{https://github.com/google-research/google-research/blob/master/dpok/pipeline_stable_diffusion_extended.py}{DPOK}.}
RLCM applies RL to few-step consistency models with Gaussian re-noising of endpoint predictions \citep{oertell2024rlcm}; our setting retains the velocity parameterization of distilled flows.

\textbf{Stochastic policies for flow RL.}
Flow-GRPO and DanceGRPO apply group-relative optimization to stochastic flow trajectories \citep{liu2025flow, xue2025dancegrpo}.
Flow-CPS demonstrates four-step training and evaluation with coefficient-preserving transitions \citep{wang2025coefficients}.
Within this family, our conditional KL criterion (Section~\ref{sec:kl_selection}) selects the existing first-order SDE-DPM-Solver++ transition \citep{lu2022dpmsolverplusplus}.
Our policy-gradient and KL weights specialize LongCat-Video's gradient-scale matching principle \citep{meituanlongcatteam2025longcatvideotechnicalreport}.

\textbf{Allocation of stochasticity and computation.}
MixGRPO confines stochastic sampling and policy updates to a moving SDE window \citep{li2025mixgrpo}; TempFlow-GRPO combines designated stochastic branching, noise-aware weighting, and seed grouping \citep{he2026tempflow}; BranchGRPO restructures rollouts \citep{li2026branchgrpo}.
Our rollout shares this mixed stochastic--deterministic structure; we address coefficient selection within a stochastic transition.

%% file: sec/6_conclusion.tex
\section{Conclusion}
In this paper, we introduce Geo-FSRL, a closed-form stochastic policy selection rule with corresponding loss reweighting for few-step flow reinforcement learning. Within a coefficient-preserving isotropic Gaussian family, the selected policy minimizes the conditional KL divergence induced by any prescribed nonzero velocity-output update, with the same minimizing coefficients across update magnitudes and directions. This separates policy selection from scale matching, which normalizes explicit loss scales without changing likelihood ratios for a prescribed output update. In both four-step text-to-image and text-to-video experiments, Geo-FSRL achieves faster reward improvement and higher final validation scores than the compared baselines, while retaining the original deterministic solver at inference. These results support treating stochastic policy construction as an explicit design choice for few-step reinforcement learning and motivate further study of its interaction with reward objectives and optimization.

\section*{Acknowledgements}
This work was supported by the National Natural Science Foundation of China (Nos. 62525601, 62476018), and the Postdoctoral Fellowship Program of CPSF (No. BX20250487).

\section*{AI use statement}

In this work, we used generative AI tools solely for language polishing and literature finding to improve presentation clarity and help retrieve related works. We have not used generative AI tools for research ideation, methodological and experimental design, or result analysis. We take responsibility for the final content of this work.

%% file: sec/experimental_details.tex
\section{Experimental Details}
\label{app:experimental_details}

Table~\ref{tab:training_hyperparameters} lists the training hyperparameters
for the GRPO-based experiments.

\begin{table}[ht]
    \centering
    \caption{Training hyperparameters for the image and video GRPO-based
    experiments. The geometry multiplier $c_{\mathrm{geo}}$ applies to the matched weights of Geo-FSRL and Flow-CPS.}
    \label{tab:training_hyperparameters}
    \small
    \begin{tabular}{lcc}
        \toprule
        Setting & Text-to-image & Text-to-video \\
        \midrule
        Optimizer updates & 4,000 & 2,000 \\
        Optimizer & AdamW & AdamW \\
        Learning rate & $2\times10^{-6}$ & $2\times10^{-6}$ \\
        Adam $(\beta_1,\beta_2)$ & $(0.9,0.999)$ & $(0.9,0.999)$ \\
        Adam $\epsilon$ & $10^{-8}$ & $10^{-8}$ \\
        Weight decay & $10^{-4}$ & $10^{-4}$ \\
        Maximum gradient norm & $0.5$ & $0.5$ \\
        Ratio clip range for $\bar r$ & $10^{-3}$ & $10^{-3}$ \\
        Reference-penalty coefficient & $3\times10^{-4}$ & $3\times10^{-4}$ \\
        Geometry multiplier $c_{\mathrm{geo}}$ & $0.5$ & $0.5$ \\
        Samples per prompt & 8 & 8 \\
        Rollout batch size per device & 8 & 4 \\
        Rollout batches per iteration & 2 & 2 \\
        Training batch size per device & 8 & 4 \\
        Optimization passes per rollout iteration & 1 & 1 \\
        Gradient accumulation steps & 1 & 1 \\
        Advantage clipping & $[-5,5]$ & $[-5,5]$ \\
        Random seed & 42 & 42 \\
        \bottomrule
    \end{tabular}
\end{table}

Flow-GRPO uses noise level $0.7$ for main image comparisons and $0.3$ for video, and caps the injected transition standard
deviation at $0.45$ (Appendix~\ref{sec:app1}). Flow-CPS uses
$\phi=0.7$ for the main image comparisons and $\phi=0.3$ for video;
the image single-reward scan also tests $0.3$, $0.5$, and $1.0$.
Flow-CPS and Geo-FSRL both use $c_{\mathrm{geo}}=0.5$ in their matched
weights. The policy-only ablation sets $w^{\mathrm{policy}}=1$ and
retains each method's $w^{\mathrm{KL}}$.

For each reward, we subtract the prompt-group mean and divide by the
largest prompt-group standard deviation in the collected rollout batch,
with a stabilizer of $10^{-4}$. In multi-reward training, the resulting
reward-specific advantages are summed with equal weights and clipped
to $[-5,5]$.

\paragraph{Reported rewards.}
For each task, the curves and tabulated automatic reward scores are
computed on the same fixed set of 256 validation prompts, shared across
methods and held out from training. Curve horizontal axes count optimizer
updates.
The image CLIP score uses CLIP ViT-L/14 matching logits divided by 30;
it is a scaled similarity and need not lie in $[0,1]$.
For video, HPSv3-General averages frame scores using the generic prompt
``A high-quality image.'' HPSv3-Percentile uses the video caption and
averages the top 30\% of frame scores. For VideoAlign-MQ, we increase the
number of sampled frames per second to improve stability, similar to previous work \citep{xue2025dancegrpo}. the  All reported automatic metrics
are also optimized during training; the human comparison reports descriptive preferences, with shared images and evaluators across ratings.

\subsection{Timestep Schedule and Coefficients}
\label{app:time_schedule}
Wan uses $t=5u/(1+4u)$ for
$u\in\{1,3/4,1/2,1/4,0\}$, giving
$(1,0.9375,0.833333\ldots,0.625,0)$ for both training and evaluation.
Qwen uses four steps from the scheduler saved with its Lightning
checkpoint. The last transition to $s=0$ is deterministic and is excluded
from policy optimization. Table~\ref{tab:wan_step_geometry} compares the coefficients on the Wan grid.

Video comparisons use $\phi=0.3$ for greater geometric separation from
the optimum, whose $\phi^\star$ values are approximately
$(1,0.783653,0.783653)$. With $\phi=0.7$, $G$ is only $1.03\%$ above
$G^\star$ at the two interior-source transitions and $12.23\%$ above
at the first. This analytic comparison does not establish a reward
ordering between Flow-CPS settings.

\begin{table}[ht]
    \centering
    \small
    \setlength{\tabcolsep}{3pt}
    \setlength{\abovecaptionskip}{3pt}
    \setlength{\belowcaptionskip}{3pt}
    \caption{Analytic Wan coefficients (five decimal places).
    Flow-CPS with $\phi=0.3$ is the reported video baseline;
    $\phi=0.7$ is shown for analytic comparison.}
    \label{tab:wan_step_geometry}
    \begin{tabular}{lrrrrrrrrr}
        \toprule
        & \multicolumn{3}{c}{Selected split ($\phi^\star$)}
        & \multicolumn{3}{c}{Flow-CPS ($\phi=0.3$)}
        & \multicolumn{3}{c}{Flow-CPS ($\phi=0.7$)} \\
        \cmidrule(lr){2-4}\cmidrule(lr){5-7}\cmidrule(lr){8-10}
        $(t,s)$ & $\alpha^\star$ & $\beta^\star$ & $G^\star$
            & $\alpha$ & $\beta$ & $G$
            & $\alpha$ & $\beta$ & $G$ \\
        \midrule
        $(1,15/16)$ & 0.93750 & 0.00000 & 0.06667 & 0.42562 & 0.83532 & 0.14685 & 0.83532 & 0.42562 & 0.07482 \\
        $(15/16,5/6)$ & 0.78567 & 0.27778 & 0.17678 & 0.37833 & 0.74251 & 0.29034 & 0.74251 & 0.37833 & 0.17859 \\
        $(5/6,5/8)$ & 0.58926 & 0.20833 & 0.47140 & 0.28374 & 0.55688 & 0.77424 & 0.55688 & 0.28374 & 0.47624 \\
        \bottomrule
    \end{tabular}
\end{table}

\par
\begin{samepage}
\subsection{Training Algorithm}
\label{app:training_algorithm}

\paragraph{Sampling arrangement.}
For each rollout batch, we select one non-terminal transition index
$k\in\{0,\ldots,K-2\}$ and use the stochastic policy only at that step.
All other steps use deterministic Euler updates with the frozen sampling
model $\theta_{\mathrm{old}}$. Only the selected transition contributes
to the loss, with its trajectory's terminal reward determining the
advantage; no gradient is propagated through the rollout.
The index cycles across rollout batches, and the recorded transitions
are shuffled before optimization. The expectation in
Eq.~\ref{eq:geometry_matched_grpo_objective} is over these rollout
batches, prompts, and group trajectories.
At inference, the updated velocity model uses the original deterministic
few-step solver.

\begin{algorithm}[H]
\caption{Geo-FSRL Training}
\label{alg:geometry_grpo}
\small
\begin{algorithmic}[1]
\Require Velocity model $\theta$; reference $\theta_{\mathrm{ref}}$;
schedule $\mathcal T$; reward model $\mathcal R$; group size $M$;
scale $c_{\mathrm{geo}}$
\State Precompute optimal coefficients and weights for each non-terminal
transition using Eqs.~\ref{eq:optimal_transition_coefficients}
and~\ref{eq:geometry_matched_weights}
\For{each GRPO iteration}
    \State Freeze $\theta_{\mathrm{old}}\gets\theta$ and initialize
    $\mathcal B\gets\emptyset$
    \For{each rollout batch}
        \State Select $k$ by cycling over $\{0,\ldots,K-2\}$
        \State For each prompt, generate $M$ trajectories with
        $\theta_{\mathrm{old}}$: sample stochastically at $k$ and use
        deterministic Euler updates elsewhere
        \State Evaluate terminal rewards and construct group-relative
        advantages $A$
        \State Store $(\bm c,\bm x_{t_k},\bm x_{t_{k+1}},k,
        \log q_{\mathrm{old}}/D,A)$ for each trajectory in $\mathcal B$
    \EndFor
    \For{each optimization minibatch from shuffled $\mathcal B$}
        \State Evaluate the current log-density and current/reference
        policy means at the stored states
        \State Form $\bar r$, the dimension-averaged Gaussian reference
        KL, and the weights for the recorded timestep pairs
        \State Update $\theta$ by minimizing
        Eq.~\ref{eq:geometry_matched_grpo_objective}
    \EndFor
\EndFor
\Ensure Optimized velocity model $\theta$
\end{algorithmic}
\end{algorithm}
\end{samepage}

%% file: sec/x_appendix.tex
\section{Additional Ablation Results}
\label{sec:ablations}
Figure~\ref{fig:ablation} gives the validation curves for
Table~\ref{tab:policy_reweight_ablation}, with KL weighting retained
throughout. The Geo-FSRL--Flow-CPS score difference changes from $-0.08$
without policy weighting to $+0.15$ with matched weighting, an observed
interaction contrast of $0.23$. The matched comparison changes both the
coefficients and their corresponding numerical weights. It therefore
evaluates their combination rather than isolating coefficient changes
under fixed numerical weights; no cross-seed inference is made from this
single-run contrast.

The coefficient scan compares globally fixed splits with the
transition-wise rule, as discussed in Section~\ref{sec:method} and the
image experiments. The reported reward curves do not measure clipping
frequency or log-ratio variance, so they do not establish which mechanism
mediates the observed gains.

\paragraph{Conditional gradient-scale comparison.}
Fix the conditioning state and timestep pair $(t,s)$, and compare the two policies at $\theta=\theta_{\mathrm{old}}$, where the ratio is one and clipping is inactive. Let $\bm J=\partial\bm v_\theta/\partial\theta$ be the velocity-model Jacobian. For method $m\in\{\mathrm{CPS},\mathrm{Geo}\}$, the gradient of the negative, dimension-averaged policy surrogate without policy reweighting is
\begin{equation}
    \bm g_{\mathrm{policy}}^{(m)}
    =-\frac{A}{D}\frac{\kappa_{t,s}^{(m)}}{\alpha_{t,s}^{(m)}}
       \bm J^\top\bm\xi
    =\frac{A G_{t,s}^{(m)}}{D}\bm J^\top\bm\xi,
    \label{eq:app_ablation_policy_gradient}
\end{equation}
where the action and advantage are held fixed when differentiating, $\bm\xi$ is the standardized residual under the old policy, and the last equality uses $\kappa_{t,s}^{(m)}<0$. Holding $A$, $\bm\xi$, $\bm J$, and $D$ fixed across the two policies therefore gives
\begin{equation}
    \bm g_{\mathrm{policy}}^{(\mathrm{CPS})}
    =\frac{G_{t,s}^{\mathrm{CPS}}}{G_{t,s}^{\star}}
      \bm g_{\mathrm{policy}}^{(\mathrm{Geo})}.
    \label{eq:app_ablation_gradient_comparison}
\end{equation}
Since $G_{t,s}^{\mathrm{CPS}}\geq G_{t,s}^{\star}$, the Flow-CPS gradient norm is no smaller than that of Geo-FSRL in this controlled comparison; the inequality is strict when the geometry factors differ and the common gradient is nonzero. This statement concerns a fixed transition. Matching training hyperparameters alone does not hold the sampled advantages, residuals, or model Jacobians fixed, and the statement does not order the norm of a batch gradient mixing different transitions.

\paragraph{Relation to the reward ordering.}
For either policy, applying its matched KL weight gives
\begin{equation}
    w_{t,s}^{\mathrm{KL}}\overline D_{\mathrm{KL}}
    =\frac{c_{\mathrm{geo}}^2}{2D}
      \|\bm v_\theta-\bm v_{\mathrm{ref}}\|_2^2.
    \label{eq:app_ablation_matched_kl}
\end{equation}
Thus, retaining KL reweighting while removing policy reweighting normalizes the reference-penalty curvature in velocity space but leaves $G_{t,s}$ in the policy gradient. With the other gradient factors controlled, Flow-CPS's larger geometry factor increases the policy-gradient scale relative to this matched regularizer. This is a conditional scale explanation compatible with its slightly higher unweighted reward, rather than a measured attribution of that difference. Appendix~\ref{app:geometry_matched_scaling} details how full updates depend on the coefficients and optimizer. The ablation evaluates cross-timestep balancing together with changes in the overall policy-to-regularization scale.

\setlength{\intextsep}{4pt plus 1pt minus 1pt}
\begin{figure}[!htbp]
\centering
\setlength{\parskip}{0pt}
\setlength{\abovecaptionskip}{3pt}
\includegraphics[width=\linewidth]{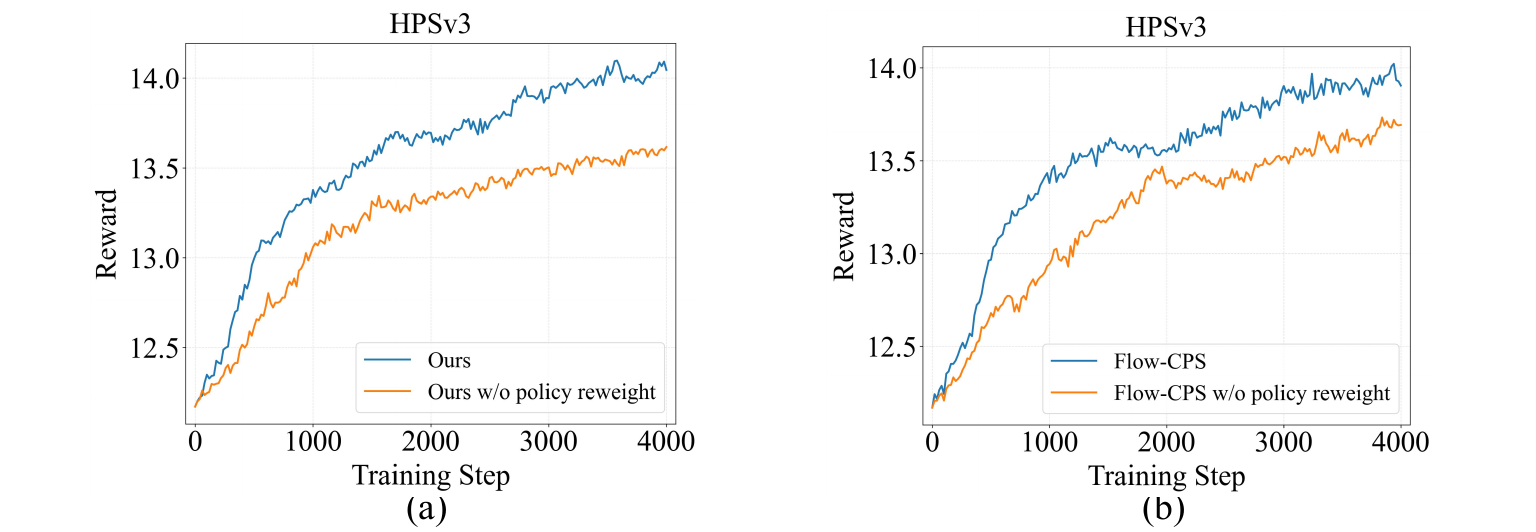}
\caption{Validation HPSv3 with and without policy reweighting for (a) Geo-FSRL and (b) Flow-CPS. All curves use the same 256 validation prompts; horizontal axes count optimizer updates. KL reweighting remains enabled throughout. Scores are reported in Table~\ref{tab:policy_reweight_ablation}.}
\label{fig:ablation}
\end{figure}

\section{Additional Qualitative Results}
\label{sec:appvis}
\begin{figure}[!htbp]
\centering
\setlength{\parskip}{0pt}
\setlength{\abovecaptionskip}{3pt}
\includegraphics[width=\linewidth]{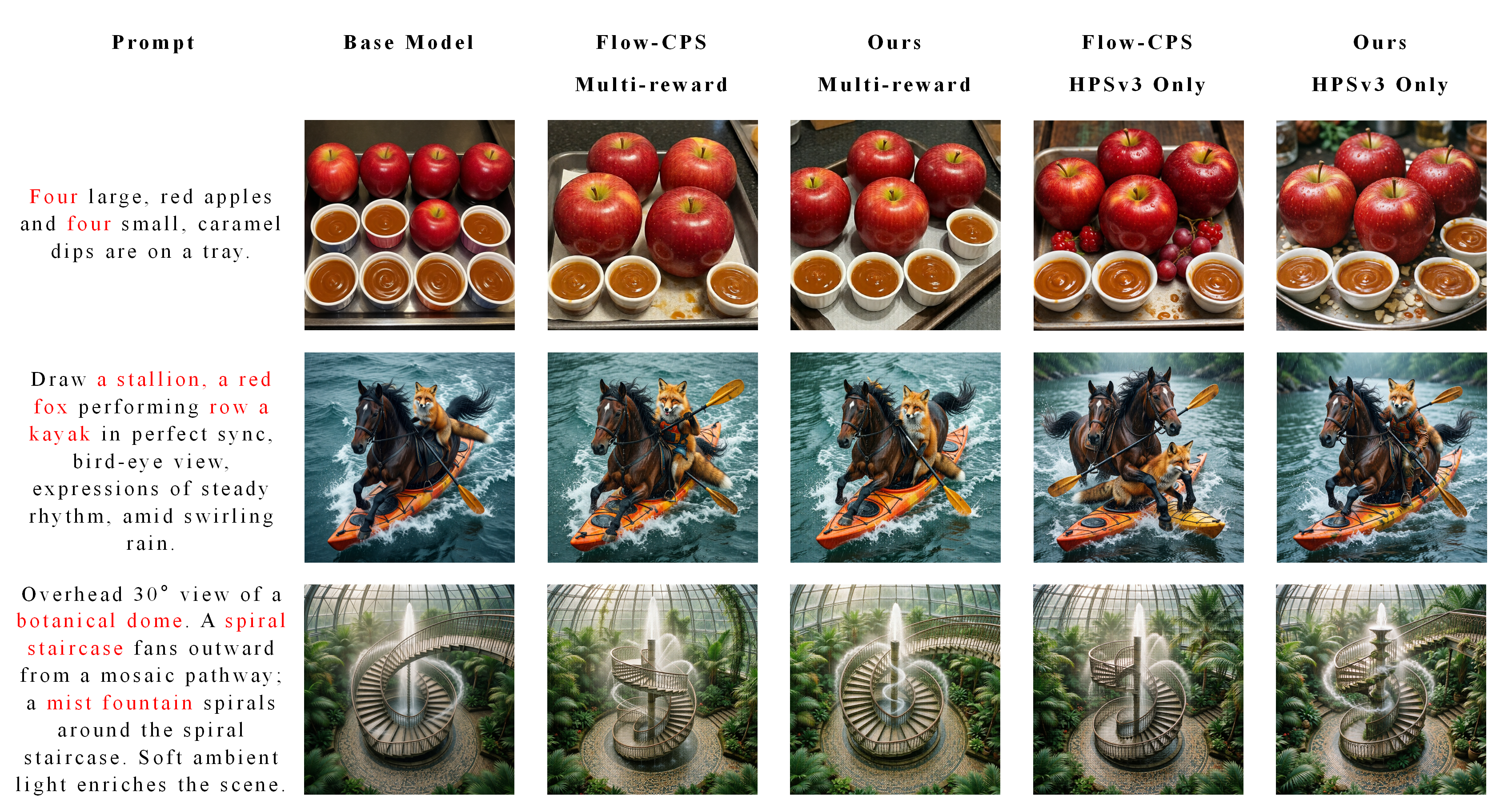}
\caption{Qwen-Image-Lightning, Flow-CPS, and Geo-FSRL on the displayed prompts. Multi-reward training optimizes HPSv3 and CLIPScore; HPSv3-only training optimizes HPSv3 alone.}
\label{fig:exp2}
\end{figure}

Figures~\ref{fig:exp4}--\ref{fig:exp5_02} compare the base models with
Geo-FSRL on additional image and video prompts.

\begin{figure}[ht]
\begin{center}
\includegraphics[width=\linewidth]{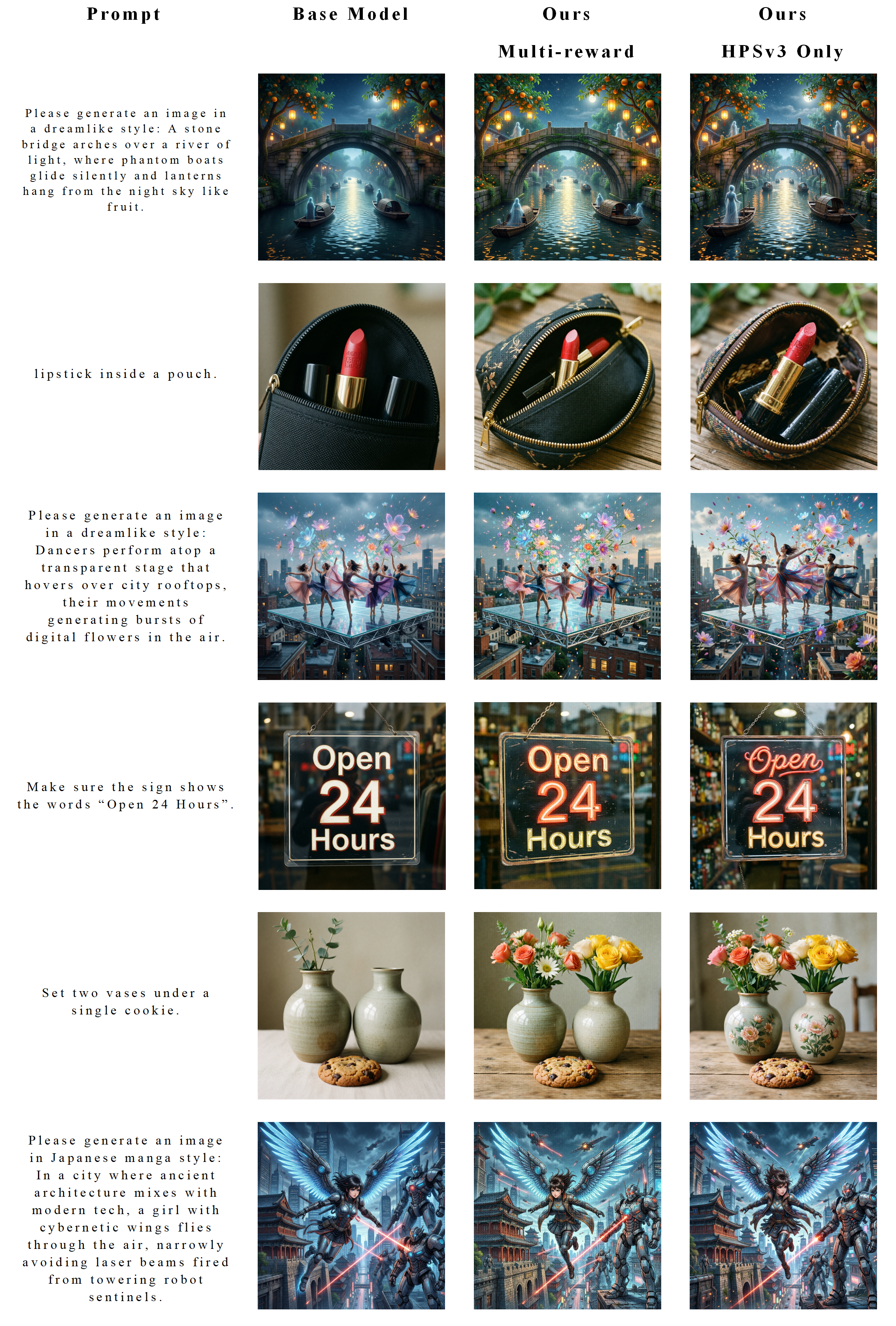}
\end{center}
\vspace{-1em}
\caption{Additional text-to-image examples from the Qwen-Image-Lightning base model and Geo-FSRL under joint HPSv3 and CLIPScore training or HPSv3-only training. Each row uses the displayed prompt.}
\label{fig:exp4}
\end{figure}

\begin{figure}[ht]
\begin{center}
\includegraphics[width=\linewidth]{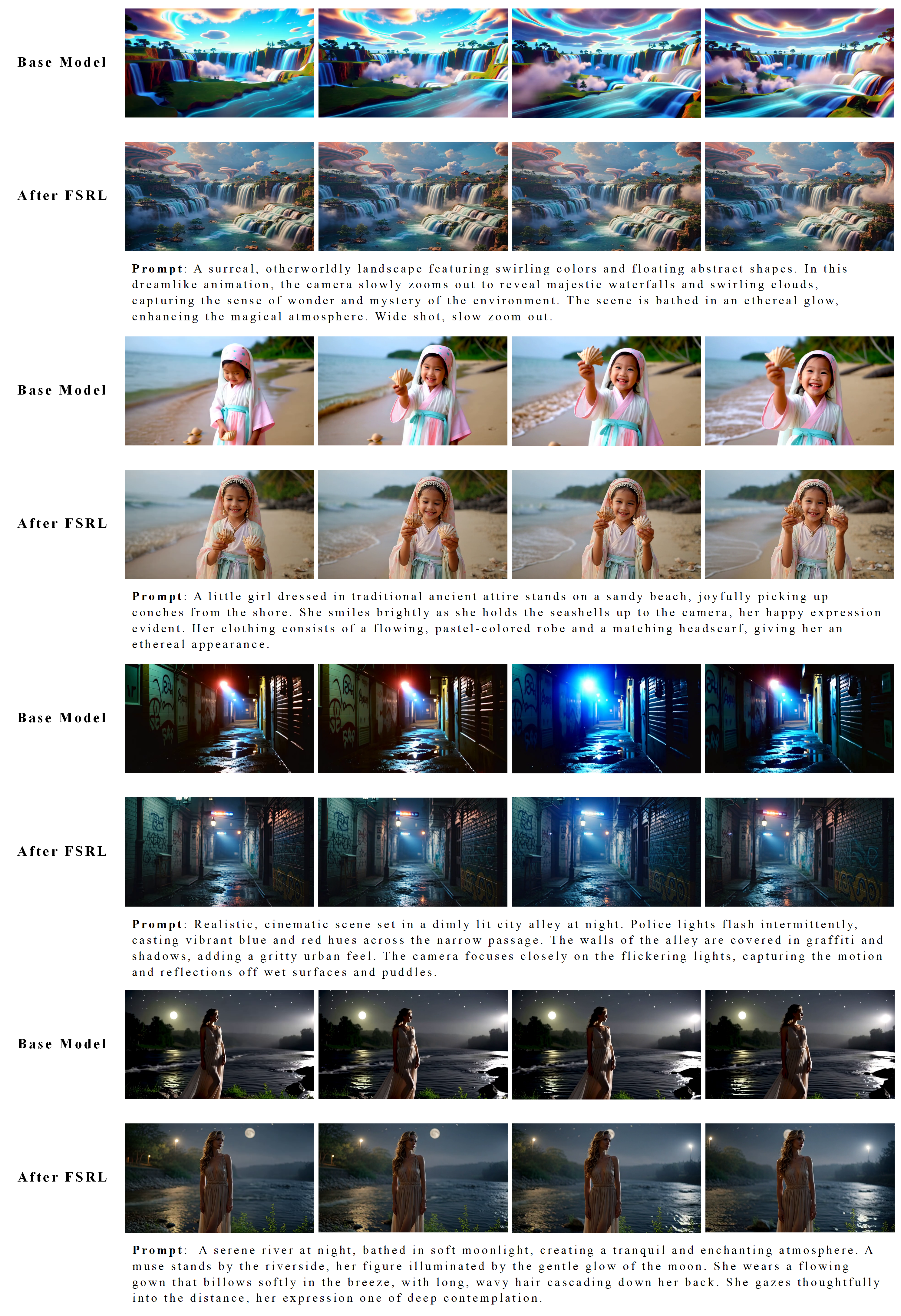}
\end{center}
\vspace{-1em}
\caption{Selected frames from videos generated by the four-step distilled Wan2.1 base model and Geo-FSRL under joint reward optimization. For each prompt, the upper row shows the base model and the lower row (``After FSRL'') shows Geo-FSRL, with four frames per video. Some prompts are shortened for display.}
\label{fig:exp5}
\end{figure}

\begin{figure}[ht]
\begin{center}
\includegraphics[width=\linewidth]{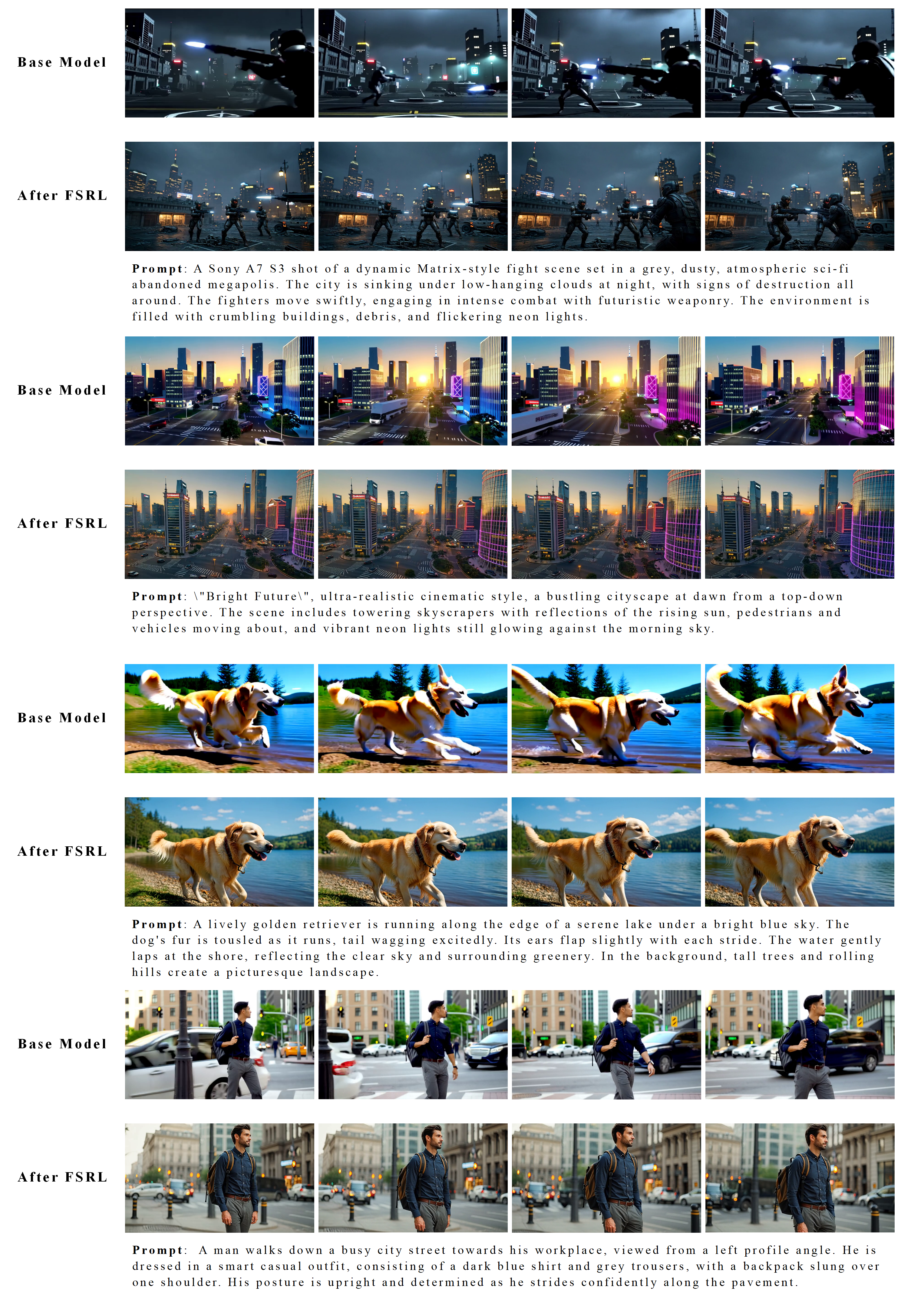}
\end{center}
\vspace{-1em}
\caption{Additional video comparisons under joint reward optimization. Each prompt has two rows of four selected frames: the four-step distilled Wan2.1 base model above and Geo-FSRL (``After FSRL'') below. Some prompts are shortened for display.}
\label{fig:exp5_02}
\end{figure}